\documentclass[12pt]{article}
\usepackage{amsmath}
\usepackage{algorithm,algorithmic}
\usepackage{times}
\usepackage{graphicx}
\usepackage{color}
\usepackage{multirow}
\usepackage{lscape}
\usepackage{subfigure}
\usepackage{enumitem}
\usepackage{mathtools}

\usepackage{bm}
\usepackage{amssymb}

\usepackage{rotating}
\usepackage{bbm}
\usepackage{latexsym}
\usepackage{fancyhdr}
\usepackage{natbib}

\DeclareMathOperator*{\argmax}{argmax}
\DeclareMathOperator*{\argmin}{argmin}
\DeclareMathOperator*{\argsup}{argsup}

\newtheorem{theo}{Theorem}[section]

\newtheorem{lemm}{Lemma}[section]
\newtheorem{proof}{Proof}

\newtheorem{defi}{Definition}[section]

\newcommand{\captionfonts}{\normalsize}

\makeatletter  
\long\def\@makecaption#1#2{%
  \vskip\abovecaptionskip
  \sbox\@tempboxa{{\captionfonts #1: #2}}%
  \ifdim \wd\@tempboxa >\hsize
    {\captionfonts #1: #2\par}
  \else
    \hbox to\hsize{\hfil\box\@tempboxa\hfil}%
  \fi
  \vskip\belowcaptionskip}
\makeatother   

\begin{document}
\hspace{13.9cm}1

\ \vspace{20mm}\\

{\begin{center}
{\LARGE Modal Principal Component Analysis}

\ \\
{ \large Keishi Sando\\
University of Tsukuba.}\\
{ \large Hideitsu Hino (hino@ism.ac.jp)\\
The Institute of Statistical Mathematics.}
\end{center}}
{\bf Keywords:} 
principal component analysis, robust statistics
\thispagestyle{fancy}
\rhead{}
\lhead{}
\chead{Draft of the paper appear in \lq\lq Neural Computation\rq\rq}
\begin{center} {\bf Abstract} \end{center}
Principal component analysis (PCA) is a widely used method for data processing, such as for dimension reduction and visualization. Standard PCA is known to be sensitive to outliers, and thus, various robust PCA methods have been proposed. It has been shown that the robustness of many statistical methods can be improved using mode estimation instead of mean estimation, because mode estimation is not significantly affected by the presence of outliers. Thus, this study proposes a modal principal component analysis (MPCA), which is a robust PCA method based on mode estimation. The proposed method finds the minor component by estimating the mode of the projected data points. As theoretical contribution, probabilistic convergence property, influence function, finite-sample breakdown point and its lower bound for the proposed MPCA are derived. The experimental results show that the proposed method has advantages over the conventional methods.

\section{Introduction}
\label{sec:introductionandrelatedwork}
Principal component analysis~(PCA;~\citet{Jolliffe2002}) is one of the most popular methods used to find a low-dimensional subspace in which a given dataset lies. Classical PCA~(cPCA) can be formulated as a problem to find a subspace that minimizes the sum of squared residuals, but squared residuals make PCA vulnerable to outliers. A lot of PCA algorithms have been proposed to robustify cPCA. The R1-PCA proposed by \citet{Ding2006} replaced the sum of squared residuals in cPCA with the sum of unsquared ones. The optimal solution of R1-PCA has similar properties to those of cPCA, that is, it is given as the eigenvectors of the weighted covariance matrix and it is rotationally invariant. The absolute residuals can reduce negative impact of outliers, but an arbitrary large outlier can still break down the estimate. More recently, \citet{Zhang2014} and \citet{Lerman2015} relaxed the optimization problem so that the set of projection matrices is extended to a set of convex set of matrices, and derived a computationally efficient robust PCA algorithm called REAPER. 
Methods proposed in \citep{Ding2006,Zhang2014,Lerman2015} perform centering and subspace estimation independently. On the other hand, we can consider performing those operations simultaneously as proposed in~\citep{Nie2014}, and our proposed method adopts this strategy.

The sum of absolute deviation is another objective function for achieving robustness~\citep{Kwak2008,Brooks2013}. In \citep{Hubert2005}, a method based on a robust covariance matrix estimation was proposed, while coherence pursuit (CoP;~\citet{Rahmani2017}) considered the correlation between inliers and outliers. The key idea of the CoP approach is that under the assumption that inliers lie in the intersection of a low-dimensional subspace and the unit sphere, each inlier is likely to have high coherence with a large number of the other inliers. CoP takes advantage of this property in order to remove outliers. \citet{Rahmani2017} derived some theoretical conditions in which CoP works well.

There are robust PCA methods explicitly considering how outliers are distributed. \citet{Schmitt2016} developed a method based on outlier detection in high-dimensional space, while \citet{Pimentel2017} utilized the random sample consensus (RANSAC) with the subspace recovery theory. \citet{DBLP:conf/colt/XuCM10} proposed a high-dimensional robust PCA (HRPCA) based on subsamplings, while \citet{8451752} combined the classical trimmed median statistics with RANSAC for robust location estimation. Recently, a method called dual principal component pursuit (DPCP;~\citet{JMLR:v19:17-436}) has been proposed. DPCP is designed to find a subspace containing whole inliers but as few outliers as possible. 
As another line of robust PCA research, a unified framework for robustifying PCA-like algorithms was proposed in~\citep{Yang2015}. Also, a low rank and sparse matrix decomposition method has been  utilized for robust PCA~\citep{Xu2010,Candes2011}. We note that robust PCA has a long history and the related literature on this subject is vast. For many other related works, we refer to a recent review~\citep{DBLP:journals/pieee/LermanM18} for an example. 

The principal component vector in cPCA is the direction in which the variance of the projected data is maximized. Variance is sensitive to outliers and various projection-pursuit (PP) PCA methods with other robust dispersion measures are proposed~\citep{Li1985,Croux2005,Croux2007}. Our proposed method is categorized in this approach. Our method is based on mode estimation~\citep{Parzen1962}. Mode estimation has a long history and its robustness to outliers has been investigated in the literature. In this study, we develop a {\it{modal PCA (MPCA)}} algorithm, which is a robust PCA algorithm based on mode estimation. 

There is a long history of research on mode statistics~\citep{10.2307/2343439}, and the estimation of mode and its related topics are still being actively studied. For example, in the field of pattern recognition, the mean shift algorithm~\citep{DBLP:journals/tit/FukunagaH75} is one of the most popular methods for density-based clustering, which is nothing but the multiple mode estimation method~\citep{400568}. Regression towards mode~\citep{LEE1989337,KEMP201292} has been attracting the interest of many statistician and is an active area of research~\citep{Yao2014,ota2019,Sando2019InformationGO}. A recent comprehensive review on the use of mode has been provided by~\citet{doi:10.1111/insr.12340}, and our work provides a novel example of the use of mode for PCA.

It is worth noting that some of the above mentioned conventional robust PCA methods give theoretical analyses such as subspace recovery. However, there is little research from the viewpoint of robust statistics, such as influence function and breakdown point. The proposed MPCA has several desirable theoretical properties. The major contributions of our work are highlighted as follows.
\begin{itemize}
    \item We prove that the objective function of the proposed MPCA converges uniformly in probability to the ground-truth probability density function (PDF) under standard regularity conditions. We also provide its convergence rate.
    \item In robust statistics, the influence function is often used for analyzing the robustness of estimators. It allows us to quantify the effect of an outlier on the estimate. We derive the influence function of the proposed MPCA, and show that the influence of an outlier on the proposed MPCA method is smaller than its influence on the cPCA method. 
    \item We introduce a finite-sample breakdown point suitable for the principal component estimator and derive a lower bound of the breakdown point (LBBP) of the proposed MPCA. Roughly speaking, the finite sample breakdown point quantifies the number of outliers that an estimator can tolerate, and its lower bound provides a worst-case evaluation for allowable contamination. 
\end{itemize}

The rest of this paper is organized as follows. After introducing notations, Section~\ref{sec:proposed_approach} proposes the minor component estimator based on mode estimation. Theoretical properties of the proposed MPCA are derived in Section~\ref{sec:theoretical_properties}. We present an optimization algorithm for the proposed method in Section~\ref{sec:algorithm} and the experimental results in Section~\ref{sec:experiments}. The last section is devoted to drawing conclusions.

\section{Notation and Proposed Approach}
\label{sec:proposed_approach}

Let $\bm{X}_i = \left( X_{i1}\ \dots\ X_{id} \right)^{\top}$ be a random vector corresponding to the $i$-th observation. When each observation is i.i.d., we use the notation $\bm{X}=\left( X_1\ \dots\ X_d \right)^{\top}$ as a random vector and express its PDF as $f_{\bm{X}}:\mathbb{R}^d\to\mathbb{R}$. The notation $\bm{x}_i = \left( x_{i1}\ \dots \ x_{id} \right)^{\top}$ refers to the $i$-th realization of $\bm{X}_i$. By projecting a random vector $\bm{X}$ on the direction $\bm{v} = \left( v_1\ \dots \ v_d \right)^{\top}$, we obtain a random variable $\bm{v}^{\top}\bm{X} = \sum_{j=1}^{d} v_j X_j$. Its PDF is represented as $f_{\bm{v}^{\top}\bm{X}}(\; \cdot \;)$. A set of unit vectors is denoted by $\mathcal{S}^{d-1} = \left\{ \bm{v} \in\mathbb{R}^d \mid \bm{v}^{\top}\bm{v} = 1 \right\}$.

In this study, the first minor component~($\text{MC}_1$) is defined as the direction in which the scatter of data is minimized. The second minor component ${\text{MC}_2}$ is the direction orthogonal to ${\text{MC}_1}$ on which the scatter of data is minimized, and other $\text{MC}_{k}, k=3,4\dots$ are defined likewise. The principal components are defined as bases of subspace orthogonal to those spanned by minor components.

Our proposed method is based on the assumption that the projected data on the true $\text{MC}$ direction tend to concentrate at a single point. cPCA measures the degree of concentration by means of the sample variance, and hence is sensitive to outliers. Instead of the sample variance, we use the probability density value of the mode as a measure of concentration and regard the direction that maximizes the probability density of the mode as the $\text{MC}$ direction.

\begin{defi}[$\text{MC}_{k}$ estimate in MPCA]
  \label{defi:optim}
  In MPCA, $\text{MC}_{k}$ estimate $\bm{v}_k$ is defined as a solution of the following optimization problem:
 \begin{align}
\notag
    &(\hat{m}_{k},\hat{\bm{v}}_{k})
      =  \argmax_{m\in\mathbb{R},\ \bm{v}\in \mathcal{S}^{d-1} }\quad \frac{1}{N}\sum_{i=1}^{N}\phi_h\left( m-\bm{v}^{\top}\bm{x}_{i} \right)
  ,\\
  \label{eq:optim1}  
    &
    \quad \text{s.t.}\quad
        \bm{v}^{\top}\hat{\bm{v}}_{j}=0, \quad j=1\dots k-1
      .
  \end{align}
  where $\phi_h(z)$ denotes a kernel function with a bandwidth parameter $h$ and $\phi_h(z)=\phi\left( z/h \right)/h$.
\end{defi}
In this study, we use $\phi(z)=\exp\left(-z^2/2\right)/\sqrt{2\pi}$. We represent $\hat{m}_k$ as the estimate of mode of the projected variable in the direction $\text{MC}_k$.
Note that in Eq.~\eqref{eq:optim1}, replacing $\phi_h$ with the negative squared loss after setting $m=0$ results in the definition of the principal component of cPCA.

\section{Theoretical Properties}
\label{sec:theoretical_properties}
In this section, we provide three theoretical results about minor components $\bm{v}_k$. The first result concerns the relationship between the objective function of MPCA and the value of the ground-truth density function~(Theorem~\ref{theo:convergence}). Second, we derive an influence function in Theorem~\ref{theo:IF}. 
We then discuss the breakdown point and present its computable lower bound in Theorem~\ref{theo:inequality}. All of the proofs for theorems are shown in the Appendix sections for the sake of readability.

\subsection{Convergence of the Mode Estimator}
\label{sec:conv}
In this subsection, we show the convergence property of the objective function of MPCA. The following theorem ensures that under some standard assumption, the minor component and the mode on that axis obtained by MPCA coincides with those of the true probability density function.
\begin{theo}[Uniform stochastic convergence]
  \label{theo:convergence}
  Let the observed data be an i.i.d. sample from a distribution with bounded variance and finite mode. The projected PDF $f_{\bm{v}^{\top}\bm{X}}(u)$ for any direction $\bm{v}$ is assumed to be bounded and differentiable with respect to $u$. Kernel function $\phi(u)$ for mode estimation, its derivative, and first- and second-order moments are assumed to exist and be finite. The bandwidth of the kernel function decays at a certain rate in $n$\footnote{The details of the regularity condition are shown in the Appendix.}. Then, we have
  \begin{align*}
    &\sup_{(m,\bm{v})\in M\times\mathcal{S}^{d-1}} \left\lvert \frac{1}{n}\sum_{i=1}^{n}\phi_h(m-\bm{v}^{\top}\bm{X}_{i}) - f_{\bm{v}^{\top}\bm{X}}(m) \right\rvert = o_p(1)
  .
  \end{align*}
\end{theo}
Note that this theorem is a text-book example of the kernel density estimate if we consider only the supremum with respect to mode $m$. This is not the case in our problem, because we consider the supremum with respect to both $m$ and $\bm{v}$, which requires more involved treatment and results in this novel theorem.
To observe the relationship between MPCA and cPCA based on this theorem, we consider the situation in which observations follow a normal distribution~$\mathcal{N}(\bm{\mu}, \bm{\Sigma})$. The projected PDF is $f_{\bm{v}^{\top}\bm{X}}(m)=\exp\left(-\frac{(m-\bm{\mu}^{\top}\bm{v})^2}{2\bm{v}^{\top}\bm{\Sigma} \bm{v}}\right)/\sqrt{2\pi \bm{v}^{\top}\bm{\Sigma} \bm{v}}$ and the value $m=\bm{\mu}^{\top}\bm{v}$ maximizes $f_{\bm{v}^{\top}\bm{X}}(m), ^{\forall}\bm{v}\in\mathcal{S}^{d-1}$. Because $f_{\bm{v}^{\top}\bm{X}}(\bm{\mu}^{\top}\bm{v})=1/\sqrt{2\pi \bm{v}^{\top}\bm{\Sigma} \bm{v}}$, the optimization problem $\max f_{\bm{v}^{\top}\bm{X}}(\bm{\mu}^{\top}\bm{v})$ is equivalent to the problem $\min \bm{v}^{\top}\bm{\Sigma} \bm{v}$. This implies that the optimization problem~\eqref{eq:optim1} results in a cPCA problem for observations that follow a normal distribution.

We then provide the convergence rate of the objective function. 
\begin{theo}[Convergence rate]
  \label{theo:convergencerate}
  In addition to the conditions in Theorem~\ref{theo:convergence}, we assume $h_n = \mathcal{O}( n^{-\frac{1}{k}} )$, where $k>4$. Then, we have
  \begin{align*}
    &\sup_{(m,\bm{v})\in M\times\mathcal{S}^{d-1}} \left\lvert \frac{1}{n}\sum_{i=1}^{n}\phi_h(m-\bm{v}^{\top}\bm{X}_{i}) - f_{\bm{v}^{\top}\bm{X}}(m)\right\rvert = \mathcal{O}_{\text{a.co.}}\left( n^{-\frac{1}{k}} \right)
  .
  \end{align*}
\end{theo}
The notation ``a.co.'' can be found in \citep{Rao1983,Shi2009}. In this paper, the notation $Z_n = O_{\text{a.co.}}\left( g(n) \right)$ denotes that for a sequence of random variables $(Z_n)_{n\in\mathbb{N}}$, there exists $M>0$ such that $\sum_{n\in\mathbb{N}} P(\left\{ \left| Z_n \right| > M g(n) \right\}) < \infty,$ where $\lim_{n\to +\infty} g(n) = 0$.

\subsection{Influence Function}
\label{sec:influencefunction}
An influence function quantifies the dependence of the estimator on a single observation. Let us define an estimator $T$ as a functional from a set of probability measures to $d$-dimensional Euclidean space. Under the assumption that observations follow the probability measure $F$, the influence function of $T$ with $F$, $\text{IF}(u;T,F)$, is defined as
\begin{align}
  \text{IF}(u;T,F)
    &= \lim_{\epsilon\to 0}\frac{T\left((1-\epsilon)F+\epsilon \Delta_{u}\right) - T\left(F\right)}{\epsilon}
,
\end{align}
where $\Delta_{u}$ denotes the Dirac measure.

For theoretical treatment, we reformulate the optimization problem~\eqref{eq:optim1} with the functional $F$ as follows:
\begin{align}
\label{eq:functional}
    &(\hat{m}_{k}, \hat{\bm{v}}_{k})
      =  \argmax_{m\in\mathbb{R},\ \bm{v}\in \mathcal{S}^{d-1}} \; \int \phi_h(m-\bm{v}^{\top}\bm{x}) dF(\bm{x})
    \\ 
    \notag
    &
    \quad \text{s.t.}\quad
        \bm{v}^{\top}\hat{\bm{v}}_{l}=0, \quad l=1\dots k-1.
\end{align}
This is an optimization problem with respect to $m$ and $\bm{v}$, and it is difficult to derive an influence function of both parameters. In order to discuss the influence function of the minor component, we assume that
the mode of the true probability measure $F$ is $\bm{0}$ without loss of generality. This assumption transforms the problem~\eqref{eq:functional} to the following problem:
\begin{align}
\label{eq:functional2}
    &\hat{\bm{w}}_{k}
      =  \argmax_{\bm{w}\in \mathcal{S}^{d-1}} \quad \int \phi_h(\bm{w}^{\top}\bm{y})dF_{\bm{Y}}(\bm{y})
  \\
  \notag
  &
    \quad
    \text{s.t.}\quad 
         \bm{w}^{\top}\hat{\bm{w}}_{l}=0, \quad l=1\dots k-1.
\end{align}
Then, the influence function of the estimate $\hat{\bm{w}}_k$ is given as follows:
\begin{theo}
  \label{theo:IF}
 Let the ground-truth probability measure be $F$. The influence function $\text{IF}(\bm{u};\hat{\bm{w}}_{k},F)$ of the estimate $\hat{\bm{w}}_k$ of $\text{MC}_{k}$, obtained by solving problem~\eqref{eq:functional2}, for the outlier $\bm{u}\in\mathbb{R}^{d}$ is given as follows:
  \begin{align*}
    \text{IF}(\bm{u};\hat{\bm{w}}_k, F)
      &=  \left( \bm{A}_k \bm{B}_k - \bm{C}_k \right)^{-1}
      \times \left[ \bm{A}_k \bm{d}_k + \sum_{l=1}^{k-1} \bm{C}_l \text{IF}(\bm{u};\hat{\bm{w}}_l, F) \right], \; (k \geq 2),\\
    \text{IF}(\bm{u};\hat{\bm{w}}_1, F)
      &=  \left( \bm{A}_1 \bm{B}_1 - \bm{C}_1 \right)^{-1} \bm{A}_1 \bm{d}_1,
  \end{align*}
  where
\begin{align*}
    \begin{aligned}
      &
      \bm{A}_k = \bm{I} - \sum_{l=1}^{k}\hat{\bm{w}}_{l}(F) \hat{\bm{w}}_{l}(F)^{\top}
    ,\\
      &
     \;
      \bm{B}_k = h^{-3} \int \left. \frac{d^2 \phi(z)}{dz^2} \right|_{z=\frac{\hat{\bm{w}}_{k}(F)^{\top}\bm{x} }{h}} \bm{x}\bm{x}^{\top} dF(\bm{x})
    ,\\
      &
      \bm{C}_k = \hat{\bm{w}}_{k}(F)^{\top}\bm{\psi}(\hat{\bm{w}}_k(F),F) \bm{I} + \hat{\bm{w}}_{k}(F) \bm{\psi}(\hat{\bm{w}}_k(F),F)^{\top}
    ,\\
      &
      \bm{d}_k = \bm{\psi}(\hat{\bm{w}}_k(F),F) - \left. h^{-2}\frac{d\phi(z)}{dz} \right|_{z=\frac{\hat{\bm{w}}_k(F)^{\top}\bm{u}}{h}} \bm{u}
    ,\\
      &
     \quad 
      \bm{\psi}(\bm{w},F) = h^{-2} \int \left. \frac{d\phi(z)}{dz} \right|_{z=\frac{\bm{w}^{\top}\bm{x} }{h}} \bm{x} dF(\bm{x})
    .
    \end{aligned}
  \end{align*}
\end{theo}

We show a simple example of the influence functions of cPCA and MPCA. Figure~\ref{fig:IF} shows that the norm $\left\| \text{IF}\left(\bm{u};\hat{\bm{v}}_{c,1},\mathcal{N}(0,\text{diag}(2,1))\right) \right\|_{2}$ of the influence functions for the estimate $\hat{\bm{v}}_{c,1}$ of $\text{MC}_1$ obtained using cPCA~(Fig.~\ref{fig:cPCA_IF}) and the norm $\left\| \text{IF}\left(\bm{u};\hat{\bm{w}}_{1}, \mathcal{N}(0,\text{diag}(2,1)) \right) \right\|_{2}$ of the influence function for the estimate $\hat{\bm{w}}_{1}$ of $\text{MC}_1$ obtained using MPCA~(Fig.~\ref{fig:ModalPCA_IF}). The probability measure $F_{\bm{Y}}$ is set to that of the normal distribution $\mathcal{N}\left(0,\text{diag}(2,1)\right)$ and the bandwidth $h$ for MPCA is set to one. The axes represent the outlier $\bm{u}=(u_1\ u_2)^{\top}$. The larger the norm of the influence function, the larger the effect of the outlier on the estimate. The solid arrows in Fig.~\ref{fig:IF} show the $\text{MC}_2$ direction and the dotted arrows show the $\text{MC}_1$ direction. Figure~\ref{fig:cPCA_IF} shows that cPCA is highly sensitive to an outlier at any point. On the other hand, in Fig.~\ref{fig:ModalPCA_IF}, the region of outliers that have a large effect on the estimate obtained using MPCA is considerably narrower than that of cPCA. The influence function of cPCA is discussed in~\citep{Croux2005,Frank1985} in detail.

\begin{figure}[ht]
  \centering
  \subfigure[cPCA]{
    \includegraphics[width=.45\textwidth]{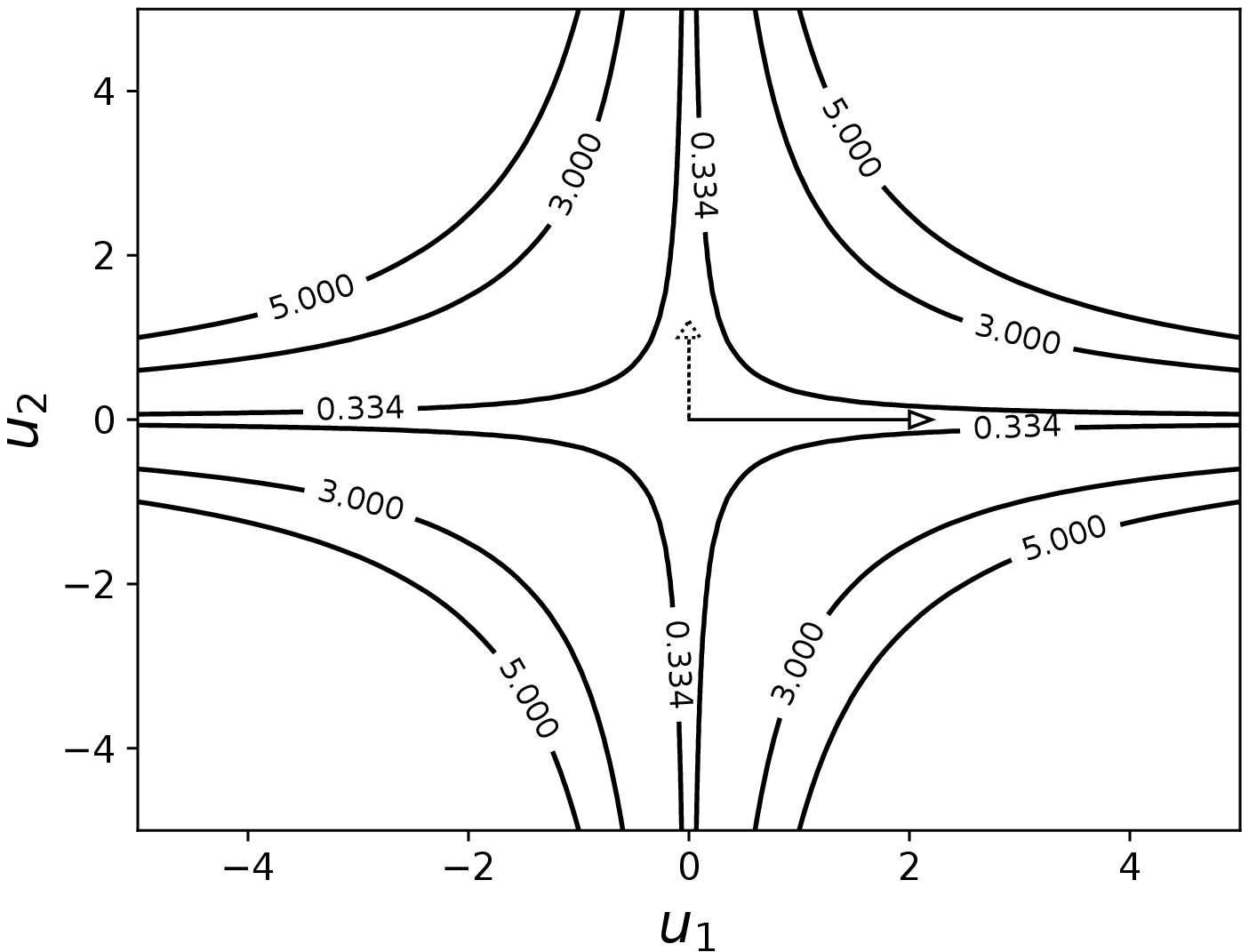}
    \label{fig:cPCA_IF}
  }
  \subfigure[MPCA]{
    \includegraphics[width=.45\textwidth]{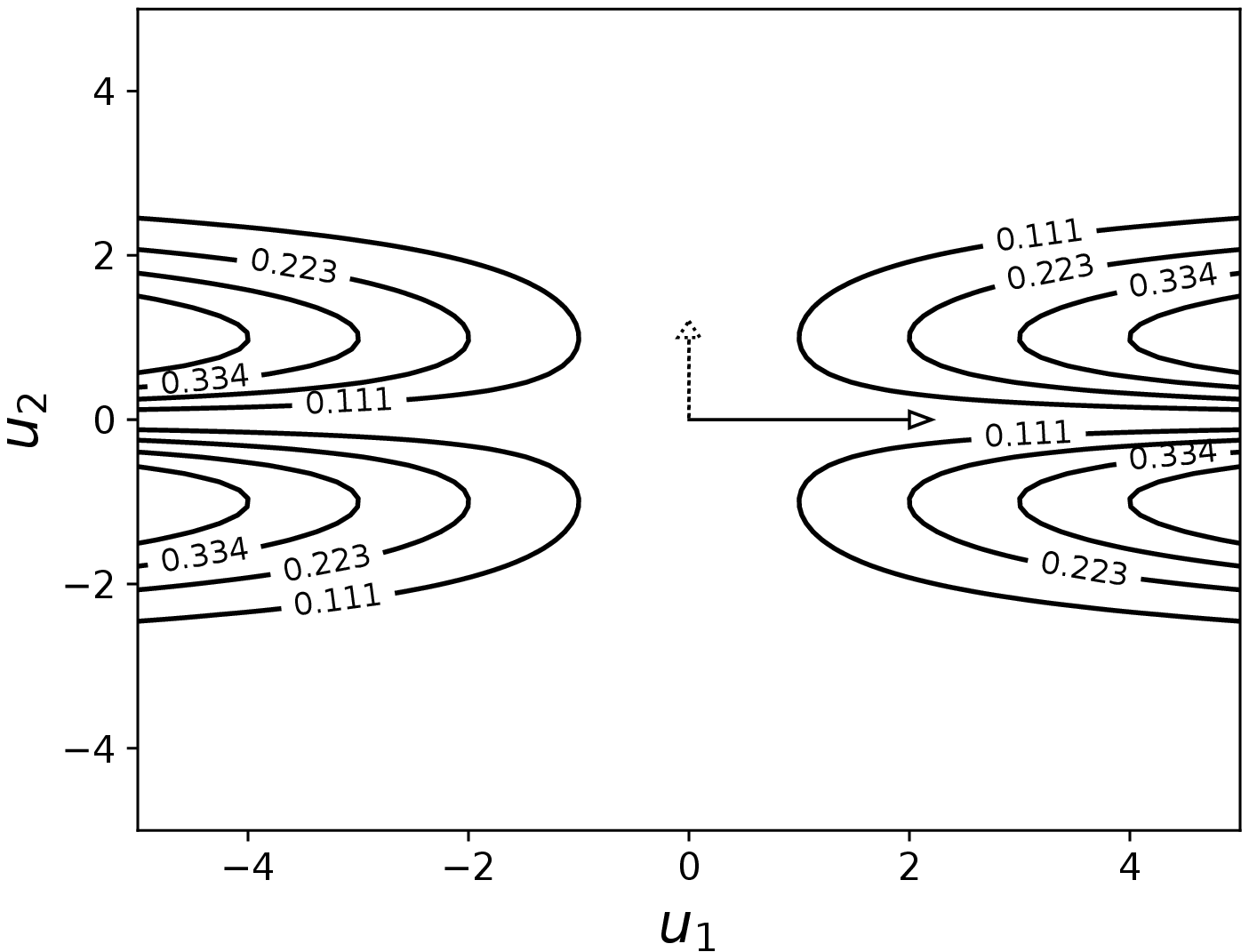}
    \label{fig:ModalPCA_IF}
  }
  \caption{Contour plot of norm of influence function for the estimate of $\text{MC}_1$ when the probability measure is $\mathcal{N}\left(0,\text{diag}(2,1)\right)$.}
  \label{fig:IF}
\end{figure}

\subsection{Breakdown Point}
\label{sec:BreakdownPoint}
Besides the influence function, the robustness of PCA can be investigated by using, for example, the breakdown point~\citep{Li1985,Croux2005}, the subspace distance~\citep{Yang2015}, and the expressed variance~\citep{Xu2013,Yang2015}. In this section, we consider the finite-sample breakdown point.

For the projection-pursuit-based methods~\citep{Li1985,Croux2005}, the breakdown point of the estimate of $\text{PC}$ is discussed via the breakdown point of the robust estimate of variance used in the projection-pursuit. However, we cannot adopt this approach, because our proposed method is not based on variance estimation.

To discuss and evaluate the robustness of PCA methods without variance estimation, such as our proposed method, we consider the angular breakdown point for unit-vector estimators for the robust discriminant analysis. 

\begin{defi}[Angular breakdown point for principal component]
  \label{defi:fsbp_eigenvector}
  Let $\hat{\bm{v}}_{k}$ be an estimator of $\text{PC}_{k}$ and $\mathcal{Y}$ be the domain of the observation. Given a set of observations $Y_{a} \subset \mathcal{Y}$, the finite-sample breakdown point~$\epsilon^{*}(\hat{\bm{v}}_{k},Y_{a})$ of $\hat{\bm{v}}_{k}$ for $Y_{a}$ is defined as
  \begin{align*}
    \epsilon^{*}(\hat{\bm{v}}_k,Y_{a})
    =  \min_b  \left\{ \frac{b}{a+b} \;\middle|\;  ^{\exists}Y_{b} \subset \mathcal{Y},\ \left| Y_{b} \right| = b,\hat{\bm{v}}_{k}(Y_{a} \cup Y_{b})^{\top} \hat{\bm{v}}_{k}(Y_{a}) = 0 \right\}.
  \end{align*}
\end{defi}

The rationale behind this definition is as follows. Given $Y_{a}$, we denote the estimate of $\text{PC}_k$ as $\hat{\bm{v}}_{k}(Y_{a})$. If we can select $Y_{b}$ and add $Y_{b}$ to $Y_{a}$ to obtain $\hat{\bm{v}}_{k}(Y_a \cup Y_{b})^{\top} \hat{\bm{v}}_{k}(Y_{a}) = 0$, then the estimate $\hat{\bm{v}}_{k}(Y_a \cup Y_b)$ can be said to be vertical to $\hat{\bm{v}}_{k}(Y_a)$. We note that similar definition of the angular breakdown point is recently proposed by~\citet{zhao2018} for linear classifiers. In their definition, the condition $\mathrm{sign} (\hat{\bm{v}}_{k}(Y_{a} \cup Y_{b})^{\top} \hat{\bm{v}}_{k}(Y_{a})) = -1$ is considered instead of $\hat{\bm{v}}_{k}(Y_{a} \cup Y_{b})^{\top} \hat{\bm{v}}_{k}(Y_{a}) = 0$. The minor components are considered to be identical up to $\pm$ sign, and our definition of the angular breakdown point for the principal component is slightly different from that used in~\citep{zhao2018}.

With $\epsilon^{\ast}(\hat{\bm{v}}_{k},Y_{a})$ in Definition~\ref{defi:fsbp_eigenvector}, we show the following theorem which gives the lower bound of the breakdown point for MPCA.
\begin{theo}[Lower bound of $\epsilon^{*}(\hat{\bm{w}}_{1},Y_{a})$ for MPCA]
  \label{theo:inequality}
  Let $\hat{\bm{w}}_{1}$ be the estimate of $\text{MC}_{1}$ obtained by solving problem~\eqref{eq:functional2}, and $\mathbb{R}^{d}$ be the domain of the observation. Given the observation dataset $Y_{a}=\left\{ \bm{x}_i \right\}_{i=1}^{a} \subset \mathbb{R}^{d}$, the following inequality holds.
\begin{align*}
    &\epsilon^{*}(\hat{\bm{w}}_{1},Y_{a}) > \frac{b^{*}}{a+b^{*}}
  ,\quad \text{where}\\
  & b^{*}=\lceil M_{a}(\hat{\bm{w}}_1(Y_{a})) - M_{a}^{*}(\hat{\bm{w}}_{1}(Y_{a})) \rceil - 1,\\
        &M_{a}(\hat{\bm{w}}_1(Y_{a})) = h\sqrt{2\pi} \sum_{i=1}^{a} \phi_h(\hat{\bm{w}}_{1}(Y_{a})^{\top}\bm{x}_i)
      ,\\
        &M_{a}^{*}(\hat{\bm{w}}_{1}(Y_{a}))
        = \sup \left\{
            h\sqrt{2\pi}\sum_{i=1}^{a}\phi_h(\bm{w}^{\top}\bm{x}_i) 
            \;
            \middle|
              \bm{w}\in \mathcal{S}^{d-1},
              \; \bm{w}^{\top}\hat{\bm{w}}_{1}(Y_{a})=0 \right\}
      .
  \end{align*}
\end{theo}
In Theorem~\ref{theo:inequality}, $M_{a}(\hat{\bm{w}}_{1}(Y_a)), M_{a}^{*}(\hat{\bm{w}}_{1}(Y_a))$ can be regarded as the number of data that are concentrated at the mode estimated by using $\hat{\bm{w}}_1, \hat{\bm{w}}_2$, respectively. The value $b^{*}$ is a lower bound of the number of outliers the dataset $Y_a$ is acceptable. Intuitively, when the degree of concentration by $\hat{\bm{w}}_{1}(Y_a)$ is larger than that by $\hat{\bm{w}}_{2}(Y_a)$, the value $b^{*}$ become large.
We consider that the LBBP is the most important contribution of our study. The derived LBBP can be computed using the given dataset. Thus, for example, we can calculate the LBBP using carefully obtained data without outliers or with only few outliers by preliminary experiments, and then estimate the tolerance of the obtained projection axes for outliers in the actual operation phase. This is possible only with the explicit formula for the LBBP.

\section{Algorithm}
\label{sec:algorithm}
In this section, we briefly describe the algorithm of the proposed MPCA. Since our main contribution is theoretically sound robust PCA, and for the sake of readability, we provide only an outline of the optimization algorithm and relegate the detailed description to the Appendix.

\subsection{Selection of Initial Point}
\label{sec:initialguess}
Since the objective function in the problem~\eqref{eq:optim1} is non-convex, it is important to select a good initial solution. We adopt the GRID algorithm proposed in the projection-pursuit robust PCA method~\citep{Croux2007}. The GRID algorithm searches the unit sphere for a better direction by multi-step grid search-like approach. \citet{Croux2007} empirically shows that GRID algorithm is able to evaluate most of possible directions efficiently.

\subsection{Optimization Problem on a Manifold}
In this subsection, we propose an algorithm to find a local optimal solution of problem~\eqref{eq:optim1} given an initial point $\bm{v}^{(0)}$. Since it is difficult to simultaneously optimize $m$ and $\bm{v}$, we solve the problem~\eqref{eq:optim1} in an iterative manner.

\mbox{}\\
\noindent{\textbf{Find a mode with fixed projection axes}}: We update $m$ by solving the following unconstrained optimization problem with respect to $m$:
\begin{align}
  m^{(l)}
    &=\argmax_{m\in\mathbb{R}} \; \frac{1}{N}\sum_{i=1}^{N} \phi_h(m-\bm{v}^{(l)\top}\bm{x}_i)
.\label{eq:optimize_m}
\end{align}
From Theorem~\ref{theo:convergence}, the objective function of this problem converges to the PDF of a projected random variable by a vector $\bm{v}$. The solution $m$ of the problem~\eqref{eq:optimize_m} is regarded as an estimate of mode, and the half-sample mode method proposed by~\cite{Bickel2006} can be used as a fast and reasonable estimator for this problem. In this study, we use the estimate $m_{(0)}$ obtained by the half-sample mode method as an initial point and apply the Newton method to obtain a higher precision solution.

\mbox{}\\
\noindent{\textbf{Optimize projection axis with fixed mode}}:
Consider the following optimization problem:
\begin{align}
  \begin{aligned}
    &\max_{\bm{v}\in\mathcal{S}^{d-1}} \quad \log \left[ \frac{1}{N} \sum_{i=1}^{N}\phi_h\left( m^{(l)}-\bm{v}^{\top}\bm{x}_{i} \right) \right]
  ,\\
    &
\quad 
\text{s.t.}\quad 
      \bm{v}^{\top}\hat{\bm{v}}_{j}=0, \quad j=1\dots k-1.
  \end{aligned}
\label{eq:optimize_v}
\end{align}
In this problem, $\frac{1}{N} \sum_{i=1}^{N}\phi_h\left( m^{(l)}-\bm{v}^{\top}\bm{x}_i \right)$ is non-negative and we can take the logarithm of the objective. By Jensen's inequality, we obtain
\begin{align}
  &\log \left[ \frac{1}{N}\sum_{i=1}^{N}\phi_h(m^{(l)}-\bm{v}^{\top}\bm{x}_{i}) \right]\notag \\ \notag
  &\geq \sum_{i=1}^{N} q_{i}^{(l)} \log \phi_h(m^{(l)}-\bm{v}^{\top}\bm{x}_{i})
        -\log N
        -\sum_{i=1}^{N}q_{i}^{(l)}\log q_{i}^{(l)},
\end{align}
where $q_{i}^{(l)}=\frac{\phi_h(m^{(l)}-\bm{v}^{(l)\top}\bm{x}_{i})}{\sum_{j=1}^{N}\phi_h(m^{(l)}-\bm{v}^{(l)\top}\bm{x}_{j})}$. Then, as a relaxation of the original problem~\eqref{eq:optimize_v}, we consider the following problem and update the estimate by its solution:
\begin{align}
  \begin{aligned}
    &\bm{v}^{(l+1)}
      =  \argmax_{\bm{v}\in\mathcal{S}^{d-1}} \; \sum_{i=1}^{N} q_{i}^{(l)} \log \phi_h(m^{(l)}-\bm{v}^{\top}\bm{x}_{i})
    ,\\
    &
    \quad 
    \text{s.t.}\quad
        \bm{v}^{\top}\hat{\bm{v}}_{j}=0, \quad j=1\dots k-1.
  \end{aligned}
\label{eq:optimize_v2}
\end{align}
We note that $\phi_h(z)=\phi\left(z/h\right)/h,\ \phi(z)=\exp\left(-z^2/2\right)/\sqrt{2\pi}$. Hence, the problem~\eqref{eq:optimize_v2} is equivalent to the following problem:
\begin{align}
  \begin{aligned}
    &\bm{v}^{(l+1)}
      =  \argmin_{\bm{v}\in\mathcal{S}^{d-1}} \quad \sum_{i=1}^{N} q_{i}^{(l)} (m^{(l)}-\bm{v}^{\top}\bm{x}_{i})^2
  ,\\
    &
   \quad
   \text{s.t.}\quad
        \bm{v}^{\top}\hat{\bm{v}}_{j}=0, \quad j=1\dots k-1.
  \end{aligned}
\label{eq:optimize_v3}
\end{align}
We set $G^{(l)}(\bm{v})=\sum_{i=1}^{N} q_{i}^{(l)}(m^{(l)}-\bm{v}^{\top}\bm{x}_{i})^2$ henceforth. Since the domain of $\bm{v}\in\mathbb{R}^d$ is restricted to $\bm{v}\in\mathcal{S}^{d-1}$, the constrained optimization problem~\eqref{eq:optimize_v3} is formulated as an unconstrained optimization problem on a manifold $\mathcal{S}^{d-1}$. In our problem, the manifold $\mathcal{S}^{d-1}$ is a simple set and we can explicitly calculate its local coordinate $\bm{\varphi}_{0}:\mathcal{S}^{d-1}\setminus\left\{-\bm{v}_{0}\right\}\to\mathbb{R}^{d-k}$, as detailed in the Appendix.
By using this local coordinate, solving the following unconstrained problem
\begin{align}
  \bm{\beta}^{(l+1)}
    =  \argmin_{\bm{\beta}\in\mathbb{R}^{d-k}} \quad G^{(l)}(\bm{\varphi}_{0}^{-1}(\bm{\beta}))
\label{eq:optimize_v4}
\end{align}
is shown to be equivalent to solving~\eqref{eq:optimize_v3} (the proof is given in the Appendix).

\subsection{Bandwidth Selection}
\label{sec:bandwidth_selection}
The proposed method requires bandwidth selection in the kernel density estimation procedure. From Theorem~\ref{theo:convergence}, the objective function of the proposed method converges to the PDF of the projected data by $\bm{v}$. Thus, we select the bandwidth in every step of update of $\bm{v}$ based on the projected data $\left\{\bm{v}^{\top}\bm{x}_{i}\right\}_{i=1}^{N}$. 
Among a large number of bandwidth selection methods~\citep{Sheather1991,Botev2010,Terrell1990,DBLP:conf/icmla/Yamasaki19}, we adopt the method proposed by~\cite{Terrell1990} because of its computational efficiency. According to \cite{Terrell1990}, when we use the Gaussian kernel function, the kernel bandwidth parameter can be estimated as $h\approx 1.144 \hat{s} N^{-1/5}$, where $\hat{s}$ is the estimate of the scale parameter of the population and we use the median absolute deviation of the observed data.

\section{Experiments}
\label{sec:experiments}
We evaluate the performance of the proposed and conventional methods on several artificial and real datasets.
In the same manner as in~\citep{DBLP:journals/csda/Scrucca11}, we
evaluate the difference between two subspaces spanned by matrices $B^{\ast}$ and $\hat{B}$ by the spectral distance: 
\begin{align*}
\mathrm{specdist} (\hat{B},B^{\ast}) =&  \arcsin \Large(
		  \|\hat{B}(\hat{B}^{\top} \hat{B})^{-1} \hat{B}^{\top} 
		   - 
		  B^{\ast}(B^{\ast \top}B^{\ast})^{-1}B^{\ast \top}\|_{S}
\Large) ,
\end{align*}
which can be regarded as the maximum angle (closeness) between column spaces of $B^{*}$ and $\hat{B}$. Here, the spectral norm $\left\| A \right\|_{S}$ is calculated by the maximum singular value of the matrix $A$. 
With a simple example, we explain what the spectral distance evaluates. Suppose that column vectors of a matrix $A = (\bm{a}_1 \; \dots \; \bm{a}_k )$ are linearly independent, and a projection matrix $P_A$ onto the space spanned by the column vectors is represented by $A(A^{\top}A)^{-1}A^{\top}$. In figure~\ref{fig:exams},
\begin{align*}
  V_1 &= \begin{pmatrix}
    0 & 0\\
    1 & 0\\
    0 & 1
  \end{pmatrix}
,&
  V_2 &= \begin{pmatrix}
    0 & 0\\
    \frac{1}{\sqrt{2}} & -\frac{1}{\sqrt{2}}\\
    \frac{1}{\sqrt{2}} & \frac{1}{\sqrt{2}}
  \end{pmatrix}
,&
  V_3 &= \begin{pmatrix}
    0 & \frac{1}{\sqrt{2}}\\
    1 & 0\\
    0 & \frac{1}{\sqrt{2}}
  \end{pmatrix}
,
\end{align*}
where $V_2$ and $V_3$ are matrices in which $V_1$ is rotated $45$ degrees around the x-axis and the y-axis, respectively. Because the space spanned by column vectors of $V_1$ is the same as that of $V_2$, $\text{specdist}(V_1,V_2) = 0$ holds. The construction of $V_3$ leads that the maximum angle between the space by $V_1$ and that by $V_3$ is $\pi/4$. This is consistent with $\text{specdist}(V_1,V_3) = \pi/4$. More detailed features of the spectral distance are discussed in~\citep{Meyer2000}.

\begin{figure}[ht]
  \centering
  \subfigure[Subspaces spanned by $V_1$ and $V_2$]{
    \includegraphics[width=.4\textwidth]{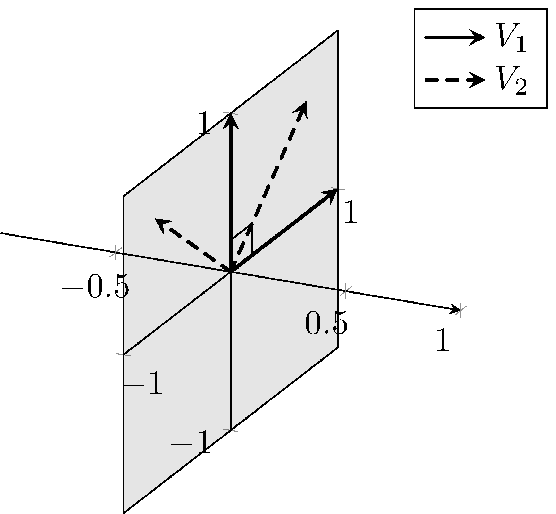}
    \label{fig:exam1}
  }
  \subfigure[Subspaces spanned by $V_1$ and $V_3$]{
    \includegraphics[width=.4\textwidth]{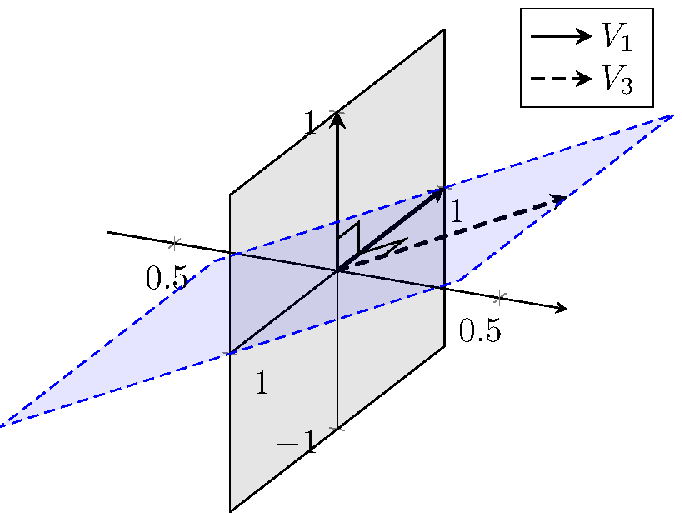}
    \label{fig:exam2}
  }
  \caption{Examples of spectral distance}
  \label{fig:exams}
\end{figure}

We compare the proposed method to classical PCA (cPCA), a projection-pursuit based method (PP;~\citet{Li1985}), a robust covariance-based method (CoP;~\citet{Rahmani2017}), and three methods based on $\ell_1$-norm for subspace identification (R1PCA;~\citet{Ding2006}, REAPER;~\citet{Lerman2015}, DPCP;~\citet{JMLR:v19:17-436}). We include the RPCAOM proposed by~\citet{Nie2014}, which performs centralization and subspace identification simultaneously, like our proposed method does. As representative of sampling-based methods, we also consider HRPCA~\citep{DBLP:conf/colt/XuCM10} and TMPCA~\citep{8451752}\footnote{All of the experiments are run on MacPro with Intel Core i7 processor and 128GB RAM.
Simple R implementation for our proposed method will be made openly available when this paper is published.}

\subsection{Artificial Dataset}
To observe the effect of the ratio of outliers, we first consider a $d=20$-dimensional random variable $X \sim \mathcal{N}(\bm{0},\Sigma_g)$ where $\Sigma_g = \text{diag}(1,1/2^2, \dots, 1/20^2)$ for the {\it{Gaussian}} case, and $X \sim Lap(\bm{1}, \Sigma_l)$ where the $j$-th dimension of $Lap(\bm{1},\Sigma_l)$ has a Laplace distribution with PDF $\exp(-|x-1|/(10/j^2))/(2 \times 10/(j^2))$ for the {\it{Laplace}} case. For $N=200$ realizations of the random variable $X$, we replace $100 \epsilon$\% of the data with outliers drawn from a uniform distribution $\mathcal{U}[(-1, 1.5)^{20}]$. The ground-truth subspace matrix $E_{k}$ is composed of canonical unit vectors with one at only one element and zero at the others. We take $k$ to be the minimum dimension that contains $95$\% of the eigenvalues. The average and one standard deviation of $\mathrm{specdist}$ values of the PCA methods in $100$ independent run are shown in Fig.~\ref{fig:art}. 

\begin{figure*}[t!]
  \centering
  \subfigure[Gaussian with varying outlier ratios]{
    \includegraphics[width=60mm,height=60mm]{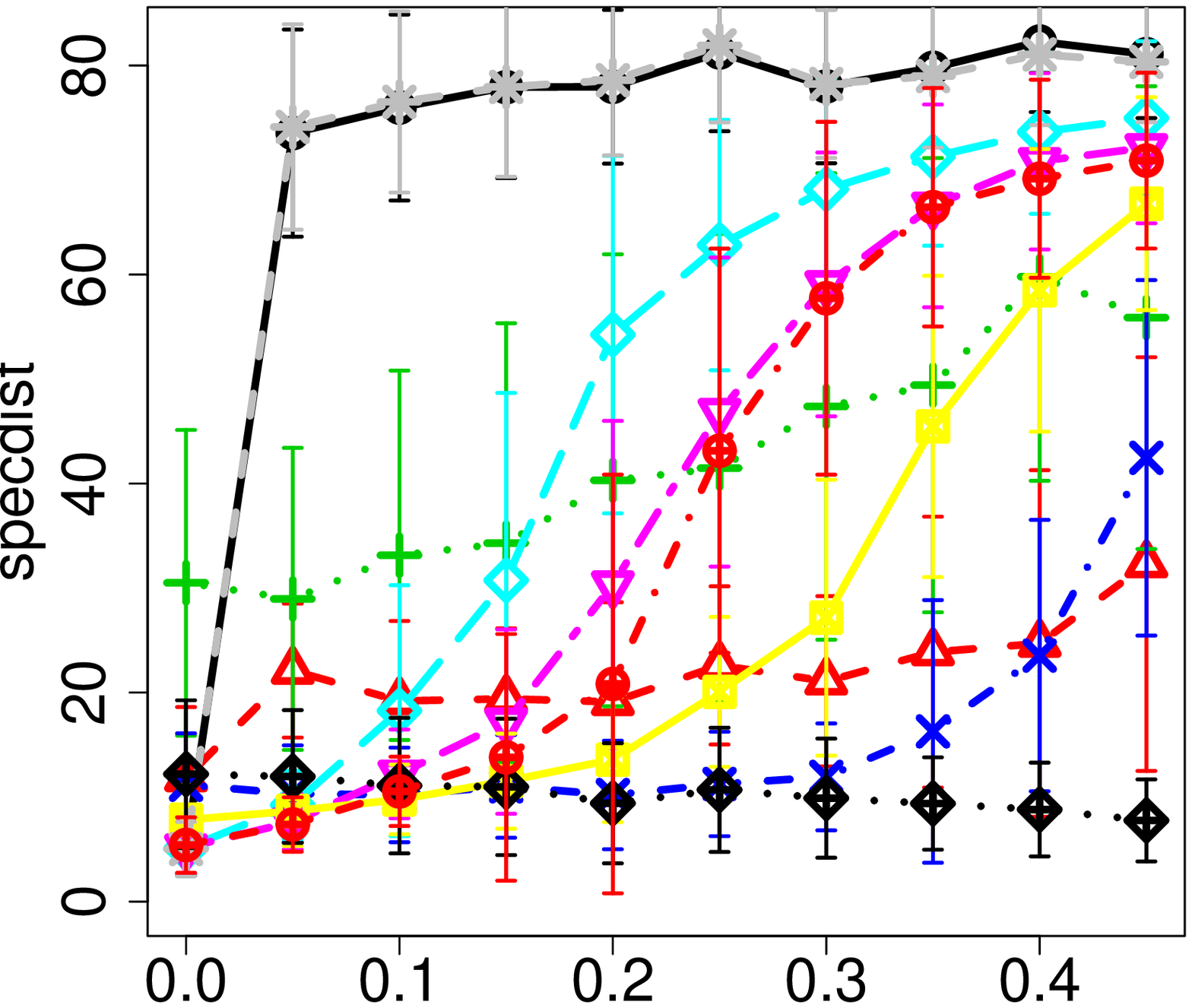}
    \label{fig:normal}
  }
  \subfigure[Laplacian with varying outlier ratios]{
    \includegraphics[width=75mm,height=60mm]{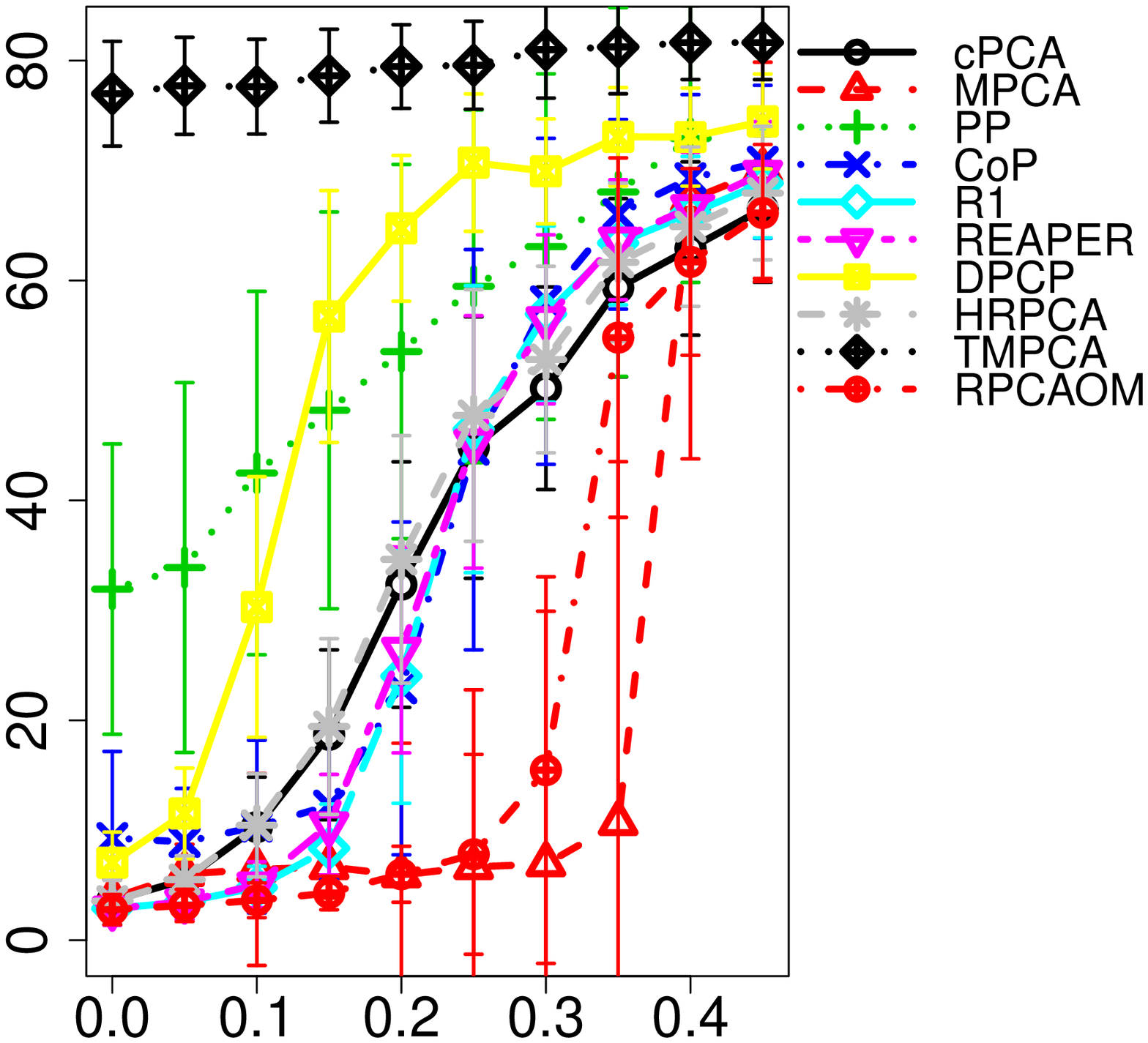}
    \label{fig:laplace}
  }
   \subfigure[Gaussian with varying \# of observations]{
    \includegraphics[width=60mm,height=60mm]{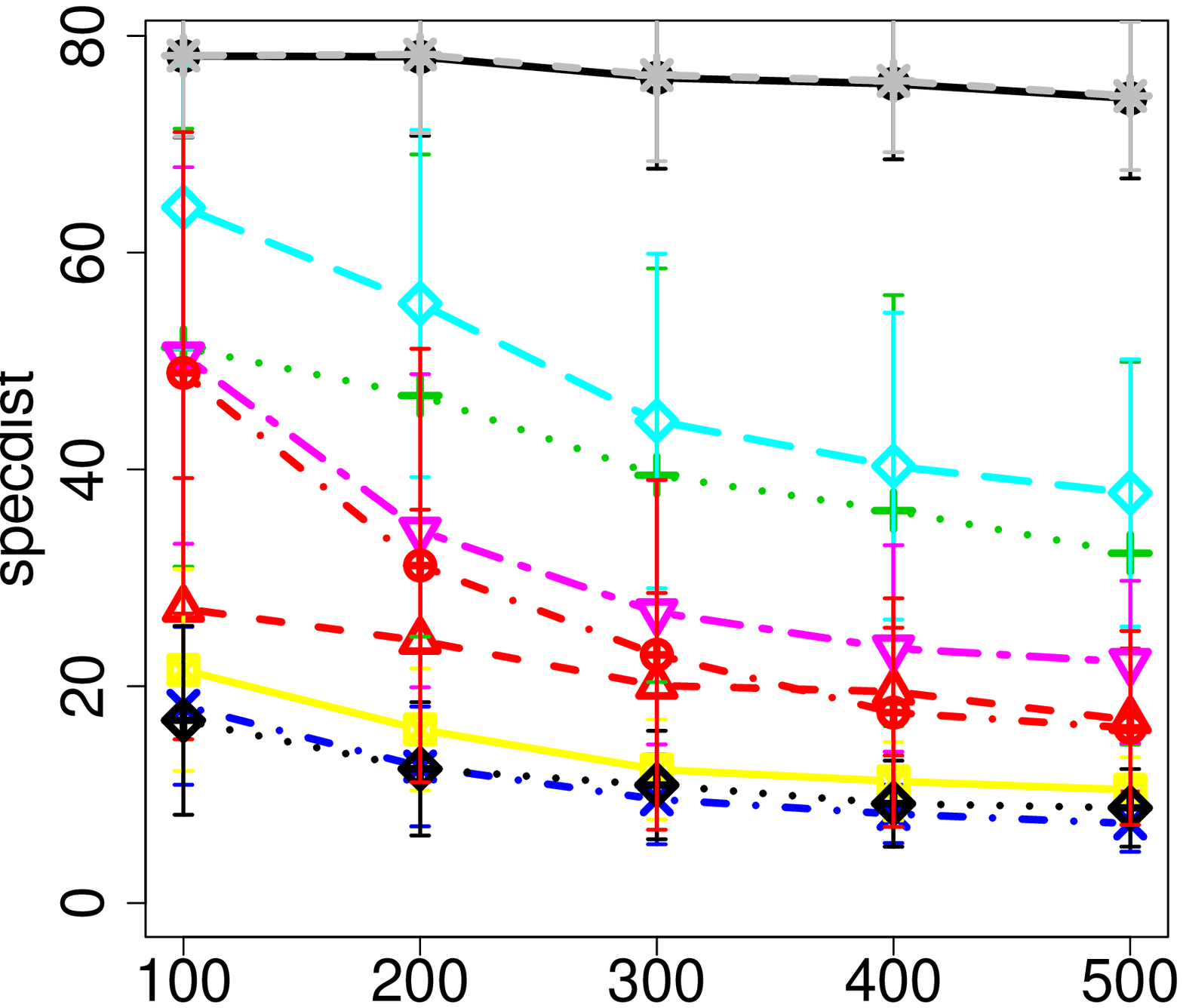}
    \label{fig:normalnum}
  }
  \subfigure[Laplacian with varying \# of observations]{
    \includegraphics[width=60mm,height=60mm]{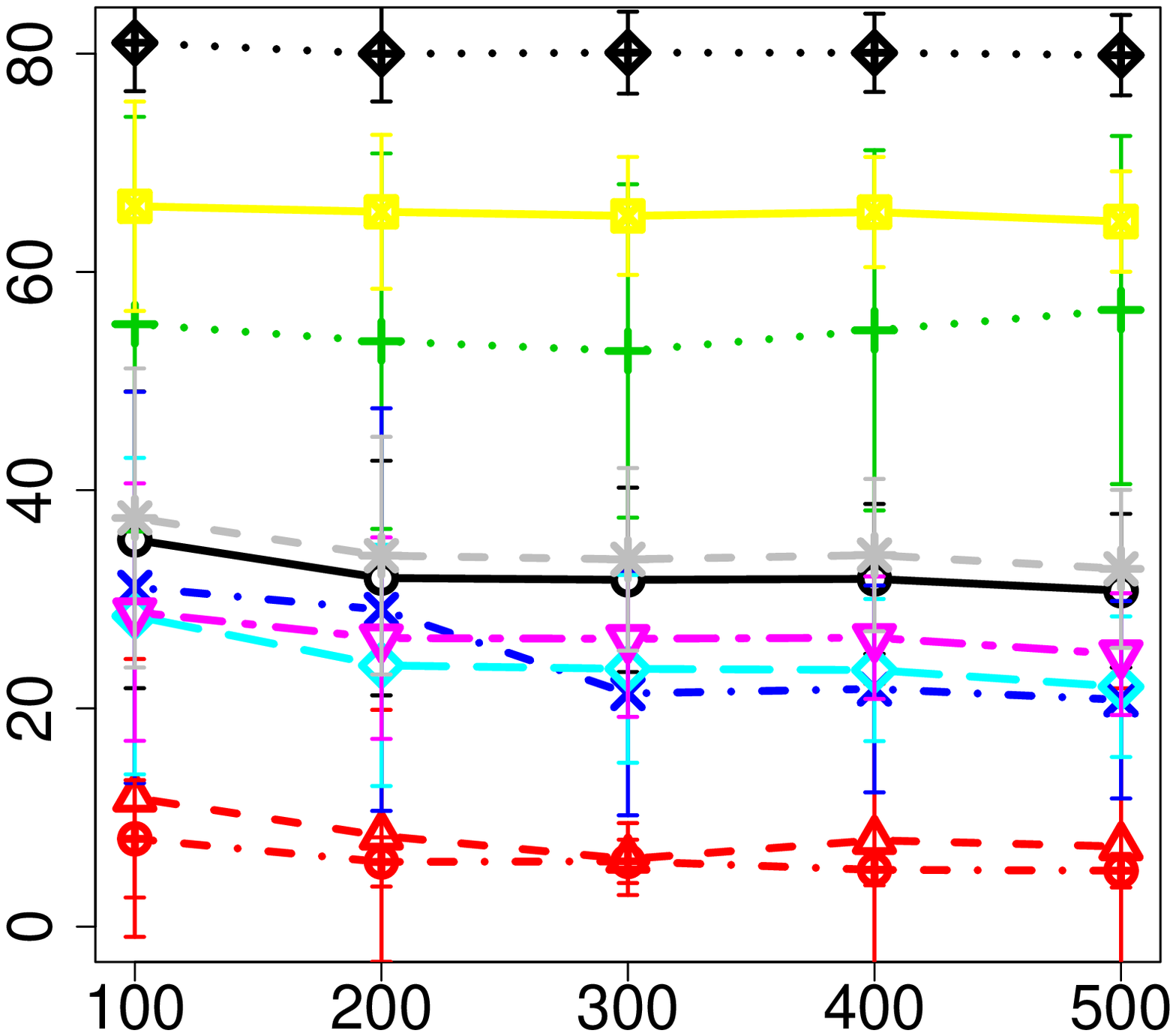}
    \label{fig:laplacenum}
  }
   \caption{The $\mathrm{specdist}$ values when the ratio of outliers varies ((a),(b)), and when the number of observations varies ((c),(d)). The inliers are from Gaussian ((a),(c)), and from Laplacian ((b),(d)).
  \label{fig:art}
  }
\end{figure*}

Figures~\ref{fig:art} (a) and (b) show that with an increase in the ratio of outliers, the performance of the estimate deteriorates. From these figures, we observe that cPCA is sensitive to a few outliers. REAPER is robust to outliers when the outlier ratio is below 0.2, but deteriorates sharply when $\epsilon > 0.2$. Similarly, DPCP and R1 show favorable performance when $\epsilon$ is very small, but rapidly deteriorate with the increase of the noise level. HRPCA and CoP are quite insensitive to outliers in the Gaussian case, but worth than other methods, including cPCA, in the Laplacian case. The proposed method is stable even with a high noise level when the distribution is the Gaussian case, and offers smallest {\rm{specdist}} value in the Laplacian case. 

We next consider the effect of the number of samples by varying the number of samples while fixing the ratio of outliers to $20$\%. The experimental result is shown in Fig.~\ref{fig:art} (c) and (d). With the increase of $N$, the performance of the most methods show improvement. PP, CoP, DPCP, and TMPCA work quite well for the Gaussian case while MPCA is the best method for the Laplacian case. Overall, the proposed method has favorable characteristics with a decrease in sample quality and an increase in sample size.

In Section~\ref{sec:bandwidth_selection}, we state that we adopt the bandwidth selection method proposed by~\citet{Terrell1990}. To observe the effect of the bandwidth selection methods, we perform a simple experiment using the above mentioned two types of artificial datasets with $n=100$ and $d=10$. We compare the spectral distances achieved by the proposed MPCA with four different bandwidth selection methods, those by~\citet{Terrell1990}(marked as {\tt{Terrell}}) in Figure~\ref{fig:art_comp_h}), \citet{Silverman86}({\tt{SRTS}}), \citet{Sheather1991}({\tt{SJ}}), and by~\citet{silverman_mode1981}({\tt{CRIT}}). The first three methods are developed for the use of the kernel density estimator, while the last method is for finding the mode of the distribution. Figure~\ref{fig:art_comp_h} shows the average and one standard deviation of specdists obtained by $100$ repetitions of estimation. In the Gaussian distribution case, {\tt{CRIT}} performs poorly and there is almost no difference between the other three methods. For the Laplace distribution case, {\tt{CRIT}} is favorable when the ratio of outliers is high, but is worse than the other three methods when the amount of outliers is not very high. Overall, there is no significant difference between the three bandwidth selection methods for KDE~({\tt{Terrell}, \tt{SRTS}, and {\tt{SJ}}}). As for the computational efficiency, {\tt{Terrell}} and {\tt{SRTS}} are comparable and {\tt{SJ}} takes twice as long as {\tt{Terrell}}. {\tt{CRIT}} is more than $100$ times slower than {\tt{Terrell}}. From this simple experimental result, we recommend using the method of~\citet{Terrell1990} or~\citet{Sheather1991} for the bandwidth selector in our proposed MPCA.

\begin{figure}[t!]
  \centering
  \subfigure[Gaussian with varying outlier ratios]{
    \includegraphics[width=70mm,height=60mm]{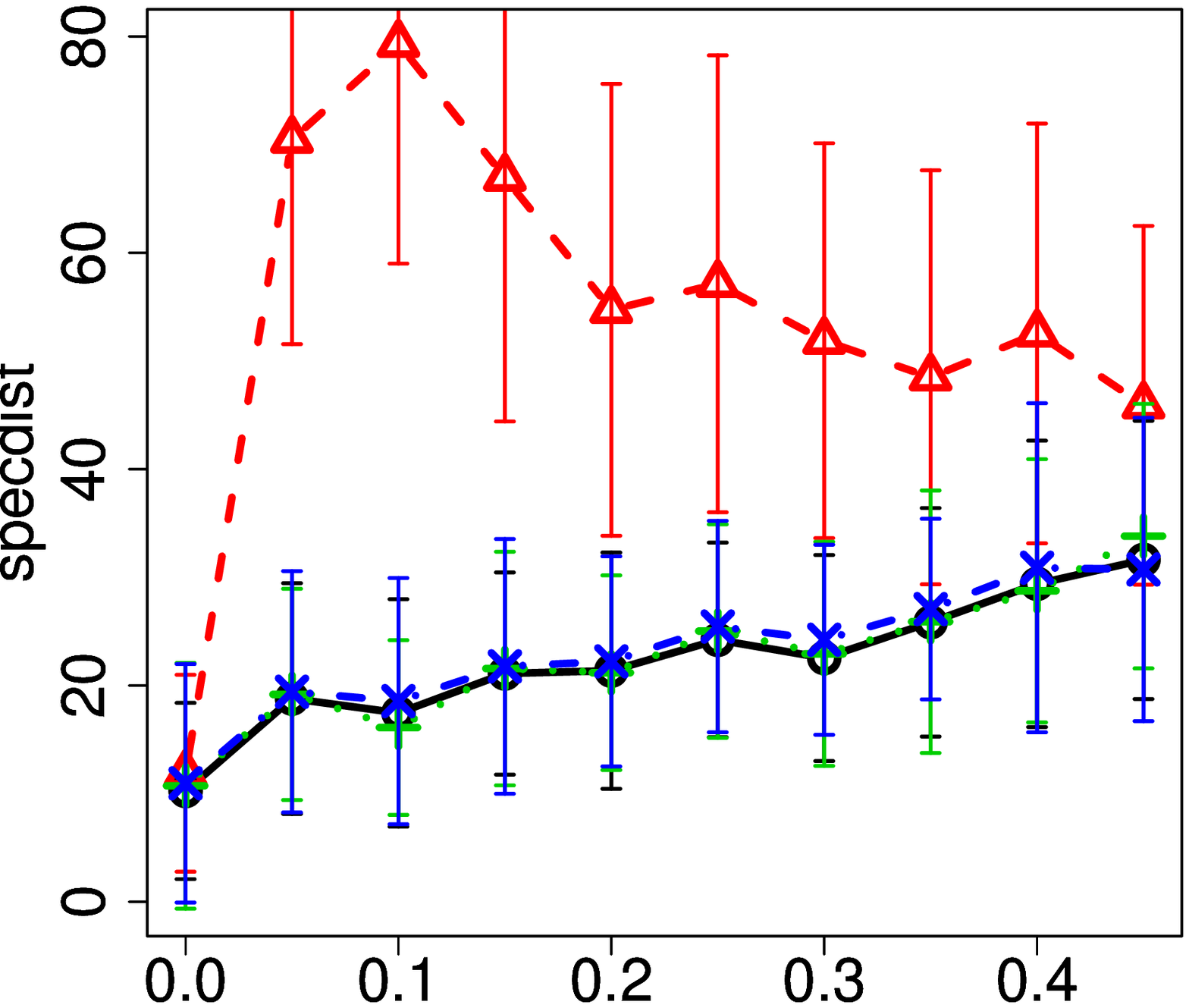}
    \label{fig:normal_comp_h}
  }
  \subfigure[Laplacian with varying outlier ratios]{
    \includegraphics[width=60mm,height=60mm]{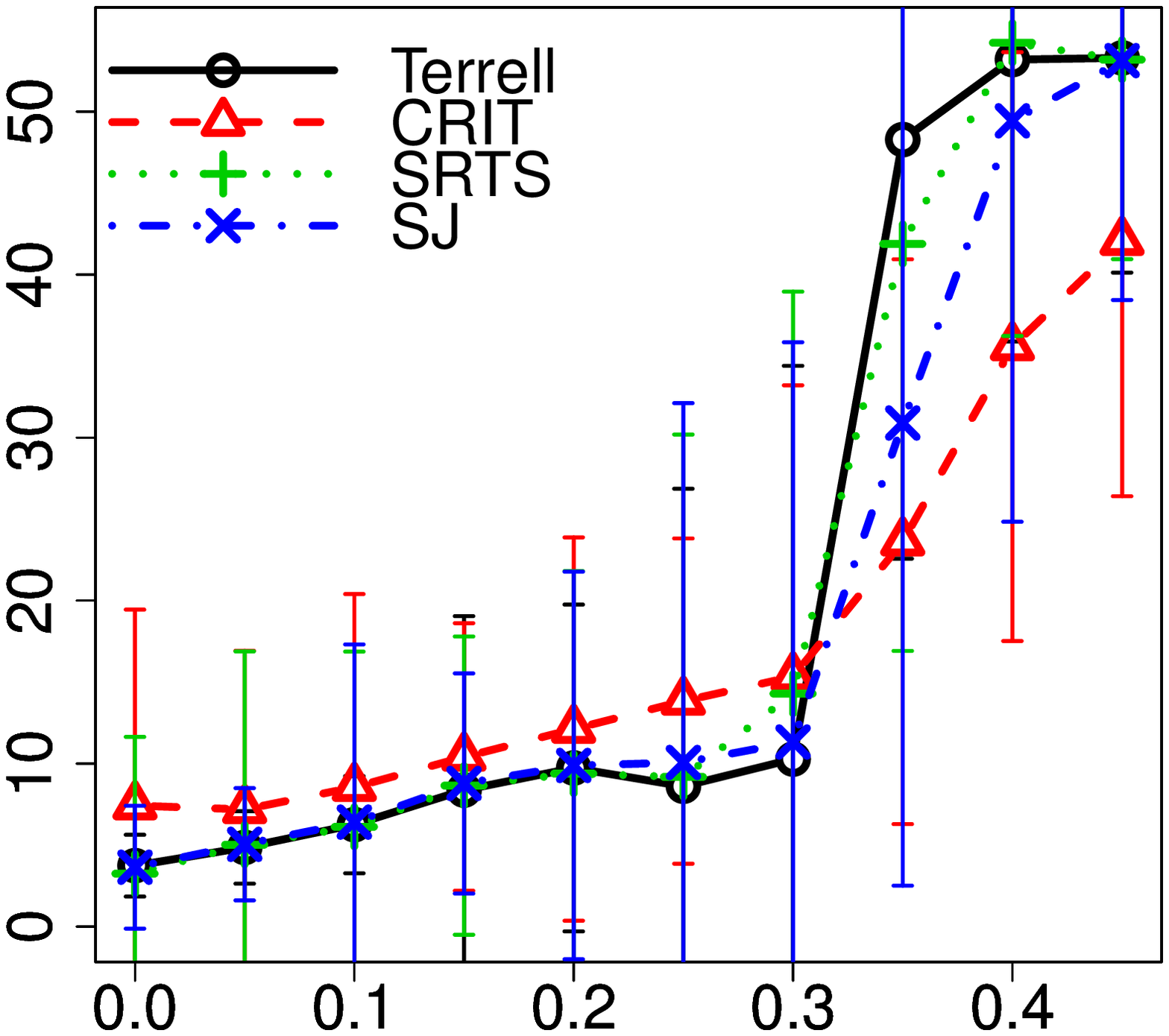}
    \label{fig:laplace_comp_h}
  }
   \caption{The $\mathrm{specdist}$ values when the ratio of outliers varies with MPCA using different bandwidth selection methods.
  \label{fig:art_comp_h}
  }
\end{figure}

\begin{landscape}
\begin{table*}[t]
  \centering
\caption{Spectral distances for real-world datasets}
\label{realresult}
 \scalebox{.8}{
 \begin{tabular}{l||c|c|c|c|c|c|c|c|c|c}

Dataset & cPCA & MPCA & PP & CoP & R1 & REAPER & DPCP & HRPCA & TMPCA & RPCAOM \\
 (in/out/dim)&&&&&&&&& \\ \hline
    wine (119/10/13) & 10.3$\pm$1.0 & 16.1$\pm$7.2 & 72.0$\pm$10.7 & 15.4$\pm$6.5 & 12.4$\pm$1.5 & 13.6$\pm$1.5 & 16.0$\pm$8.0 &10.2$\pm$8.4& 89.6$\pm$9.8 & 14.4 $ \pm$ 2.1\\
	 wbc (357/21/30) &
	 68.1$\pm$8.6 & 36.0$\pm$12.3 & 81.7$\pm$3.6 & 30.5$\pm$19.4 & 69.0$\pm$7.6 & 43.3$\pm$18.0 & 27.3$\pm$13.0 &66.8$\pm$8.7 & NA & 79.7 $ \pm$ 7.5\\
	 vertebral (210/30/6)&
	 10.8$\pm$3.3 & 6.1$\pm$14.1 & 40.0$\pm$18.7 & 24.3$\pm$13.7 & 13.0$\pm$3.4 & 10.4$\pm$2.1 & 8.7$\pm$1.9 &8.2$\pm$2.3 &59.7$\pm$18.0 & 5.9 $ \pm$ 3.9 \\
	 thyroid (3679/93/6)&
	 67.0$\pm$21.4 & 1.4$\pm$0.7 & 49.9$\pm$22.7 & 1.5$\pm$0.3 & 87.7$\pm$27.3 & 79.4$\pm$9.3 & 0.3$\pm$ 0.5& 66.7$\pm$ 21.2& NA  & 88.8 $\pm$ 27.7\\
    pendigits (6714/156/16)&
    6.1$\pm$2.18 &  5.4$\pm$3.49 & 45.2$\pm$18.56 & 5.1$\pm$0.49 & 2.5$\pm$1.42 &  1.7$\pm$0.31 & 10.9$\pm$5.11 & NA&NA & NA
\end{tabular}
}
\end{table*}
\end{landscape}

\subsection{Real-world Datasets}
We evaluate the performance of robust PCA methods on five real-world datasets obtained from the UCI Machine Learning Repository (details are explained in the Appendix). In real-world datasets, there is no ground-truth projection matrix or dimension-reduced subspace. We estimate the principal directions by using only inliers, which are regarded as the ground-truth directions. The {\rm{specdist}} between the ground-truth and those estimated with all data, including the outliers, are used to evaluate the performance of PCA methods. 

Table~\ref{realresult} summarizes the profile of datasets and median $\pm$ standard deviation of {\rm{specdist}} values evaluated by 10-fold CV. The entry of the table for which the computational time exceeded five minutes is filled by {\tt{NA}}. 
A possible explanation for the slowness of sampling-based methods is that TMPCA requires many subsamplings, and HRPCA executes many iterations with an increase of the sample size. 
There is no method that outperforms another in all cases, but the proposed method offers comparable performance in many cases and is an alternative option to conventional PCA methods when the given data seem to be contaminated by outliers.

\subsection{Discussion on the Experimental Results}

From Figure~\ref{fig:art}, we observe that the standard deviation of the estimates by MPCA tend to be larger than those of the other methods. This also holds for one of the real-world datasets, {\tt{vertebral}}. In our proposed algorithm, the failure to find a good initial estimate could make the performance worse than the other methods, which we consider explains the relatively large variance of the proposed method. Nevertheless, MPCA is stable even with high noise level when the distribution is the Gaussian case, and offers the smallest specdist value in the Laplace case. 
 
There are various robust PCA methods, and it is difficult to find a situation in which the proposed method is clearly superior to others. One of the major distinguishing features between the proposed the proposed MPCA and many other robust PCA methods is whether the center of the data space and subspace are identified separately or simultaneously.
CoP, R1, REAPER, DPCP, and HRPCA assume that data are centered, and must be used in combination with some robust centering method. In our implementation, the geometric median is used. Robust centering methods, such as geometric median, are robust in themselves, but are still affected by outliers. If this centering does not work well, the resulting robust PCA method cannot achieve the expected performance.
In addition, as Figure~\ref{fig:art} shows, the RPCAOM shows similar performance for artificial data (especially for Laplace data) as the proposed method. This finding suggests that simultaneous optimization approaches, such as MPCA and the RPCAOM, are better when centering methods, such as geometric median, do not work well owing to a large number of outliers. In summary, one of the advantages of the proposed method is its simultaneous operation of centering and subspace identification. Compared to the RPCAOM, which also perform centering and subspace identification simultaneously, the proposed method is accurate and computationally efficient, as Table~1 shows.

\subsection{Evaluation on the Lower Bound of the Breakdown Point}
We show simple experimental results on the lower bound of the breakdown point (LBBP) stated in Theorem~\ref{theo:inequality}. We consider a three-dimensional dataset, in which the first two dimensions are generated from~$\mathcal{N}(\bm{0}, \mathrm{diag}(1,0.3))$ and the third dimension is from $\mathcal{N}(0,\sigma_{z})$, where $\sigma_{z} < 0.3$. The $\text{PC}_1$ for this generative model should be $(1,0,0)$ and $\text{MC}_1$ should be $(0,0,1)$. We sample $500$ points and evaluate the LBBP value on this clean dataset. Then, we generate a contaminated sample by generating another dataset of size $500$ and replacing $500 \times \alpha$ points with outliers, where $\alpha$ is varied in $\{0.01,0.02, \dots, 0.50\}$. The first two dimensions of outliers are generated from $\mathcal{N}(\bm{0}, 0.01 \times \mathrm{diag}(1,0.3))$ and the third dimension is from $\mathcal{N}(150, \sigma_{z})$. The third coordinate of the outliers is dominant and with a small number of outliers, $\text{MC}_1$ would be orthogonal to that of inliers $(0,0,1)$. We vary $\sigma_z$ to obtain different LBBPs, because it is not easy to control the LBBP value. 

For each noise fraction $\alpha$, we calculate the cosine of $\text{MC}_1$ vector obtained from clean and contaminated datasets $100$ times by changing the seed of the random number generator. We show the results when LBBP values are $0.093$ and $0.201$ in Fig.~\ref{fig:lbbp}. From Fig.~\ref{fig:lbbp} (left), we observe that $\text{MC}_1$s estimated from clean and contaminated data could become almost orthogonal when the fraction of outliers is about $0.21$, which is larger than $0.094$. Figure~\ref{fig:lbbp} (right) shows that $\text{MC}_1$ can become almost orthogonal when the fraction of outliers is about $0.36$, which is larger than $0.201$. This experimental result is consistent with Theorem~\ref{theo:inequality}.

\begin{figure*}[!t]
  \centering
    \includegraphics[width=65mm,height=60mm]{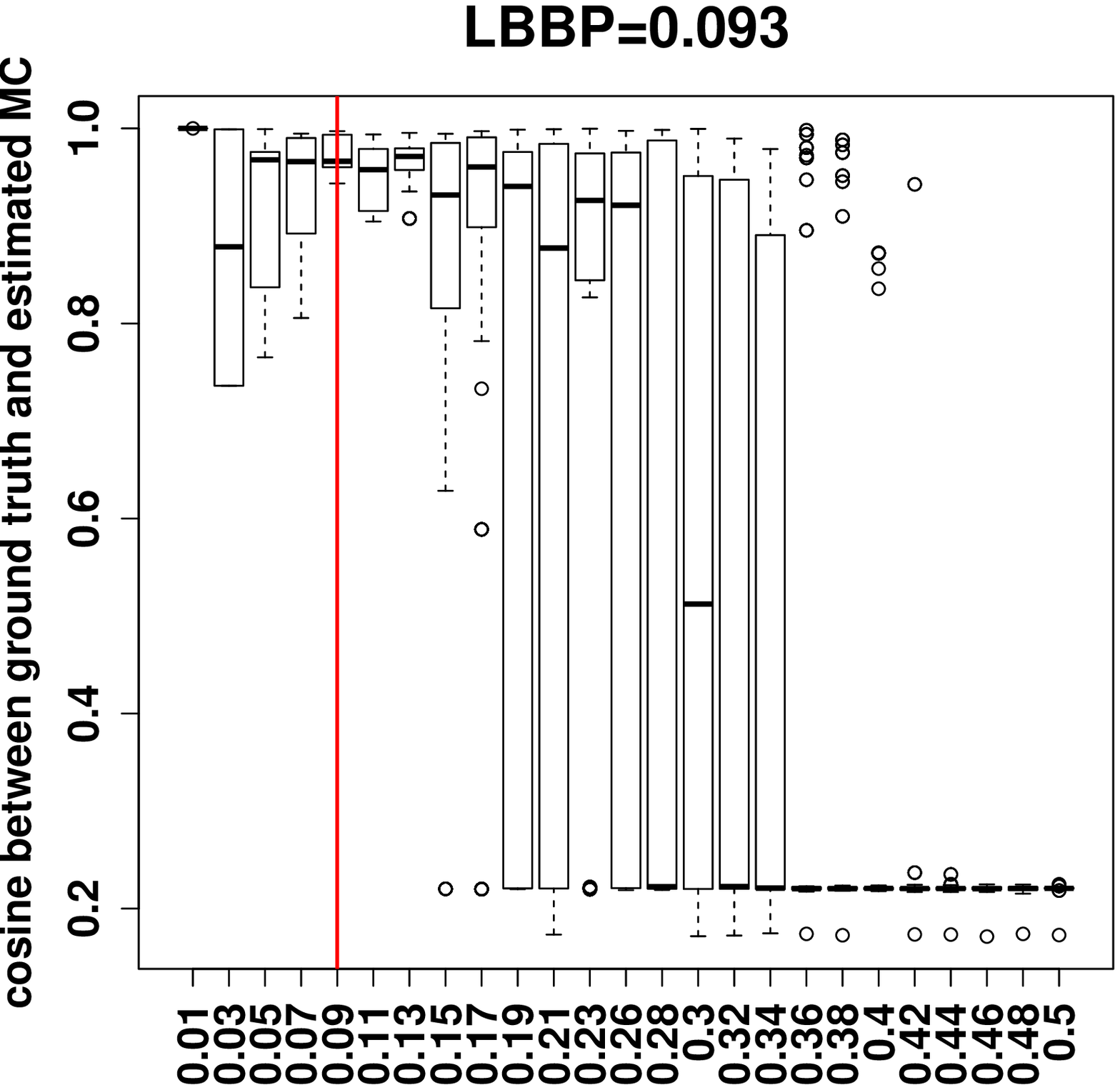}
    \includegraphics[width=65mm,height=60mm]{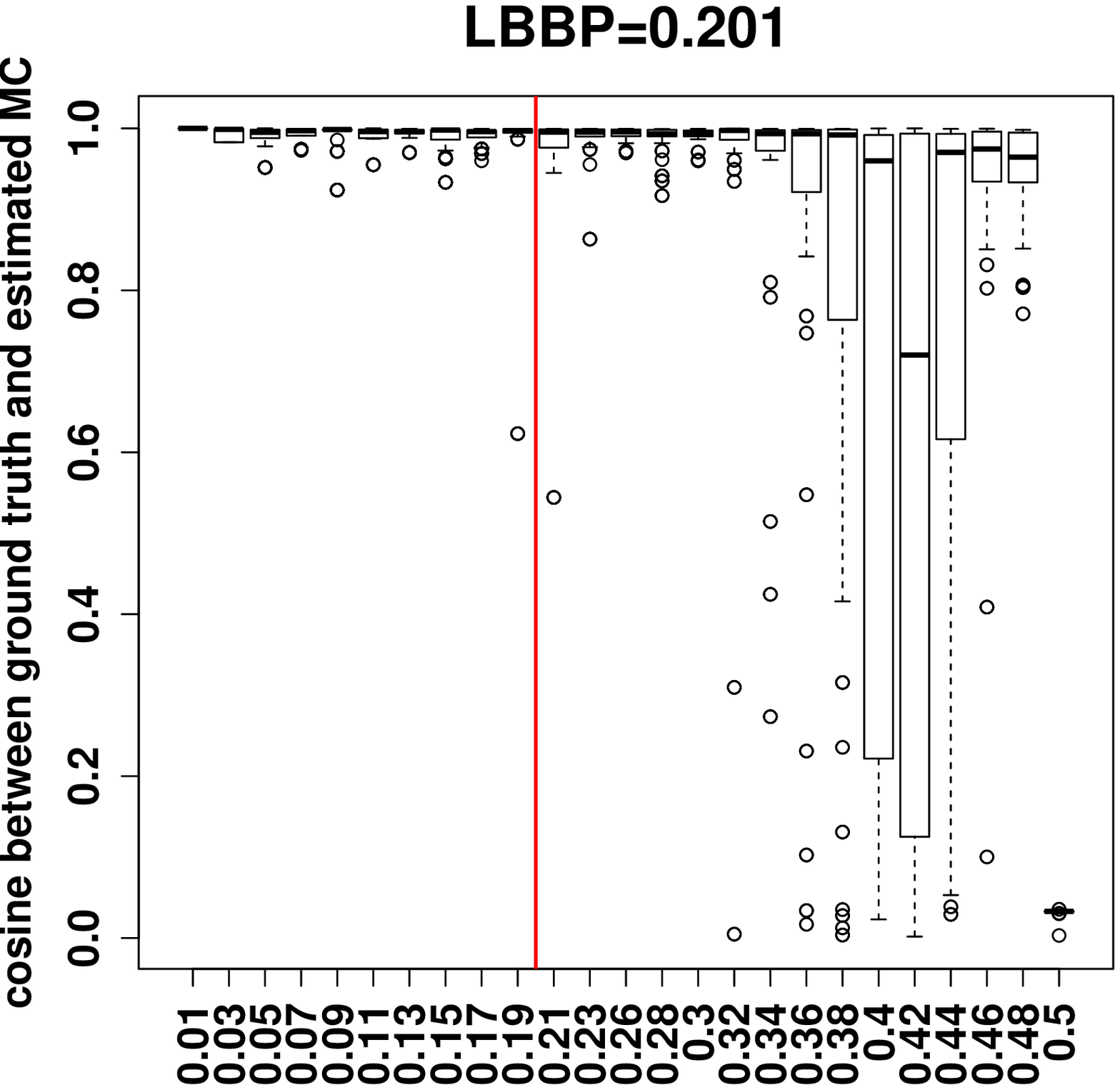}
   \caption{The cosine between $\text{MC}_1$s of noise-free and contaminated datasets 
  \label{fig:lbbp}
  }
\end{figure*}

\section*{Conclusion}
\label{sec:conclusion}
In this study, we proposed the mode-based PCA for contaminated data analysis. This method is characterized by a novel definition of a minor component. We derived the uniform convergence in probability for the proposed objective function. We provided robustness analyses of MPCA using the influence function and the breakdown point. The experiments on synthetic and real-world datasets showed comparable performance to other robust PCA methods. In addition, the derived lower bound of breakdown point is experimentally supported. The proposed method has good theoretical properties and favorable experimental results.

This study focused on theoretical properties and the basic algorithm to obtain a solution. The experiments showed that the algorithm works well. There are many important directions for future work. Scaling to more high-dimensional datasets is of practical importance. The stability of the algorithm could be improved by, for example, developing a better method for finding the initial point. Deeper theoretical analyses of the proposed method, such as asymptotic properties, is another interesting research direction. In particular, the obtained rate $\mathcal{O}_{\text{a.co.}}\left( n^{-\frac{1}{k}} \right)$ in Theorem~\ref{theo:convergencerate} is slower than the standard rate for the kernel density estimation and the mode estimation~\citep{Rao1983,Vieu1996,Shi2009} because of the complicated relation between $m$ and $\bm{v}$. Derivation of the improved rate in our setting is left for our future research.

Finally, we consider that it is of great importance to develop a method to make the best use of the derived lower bound of the break down point.
In general, it is difficult to obtain data without outliers. Many robust data analysis methods have been proposed, like the method in this study. However, there is no method that allows an arbitrary number of outliers. It is of practical importance to understand how many outliers are allowed depending on the method used for the analysis and the nature of the data being analyzed. For example, consider analysis of sensor data in an industrial facility where the cost of producing products without an outlier is quite expensive, but it is reasonable to assume that throughput is extremely high by allowing a small number of outliers. In such a situation, it is possible that efficient manufacturing and experiments can be performed by estimating the LBBP in advance and then running the facility under a precision that is consistent with the estimated allowable outlier ratio. The breakdown point itself requires the ground-truth distribution, but our proposed lower bound can be calculated only with the observed dataset. This remains for future research.

\subsection*{Acknowledgement}
H.H. is supported by JST CREST JPMJCR1761, JSPS KAKENHI JP19K12111, and JP19K21686. The authors thank Professor Hironori Fujisawa for his useful comments on the preliminary version of this manuscript.

\section*{Appendix A: Proof of Theorem~\ref{theo:convergence}}
We give a proof of uniform stochastic convergence of Theorem~\ref{theo:convergence}. The theorem is based on the following assumptions, all of which are standard regularity conditions often placed for giving theoretical guarantee for non-parametric estimators.

\begin{enumerate}[label=A.\arabic*]
  \item $\displaystyle \left\{ \bm{X}_n \right\}_{n\in\mathbb{N}}$ are i.i.d.\label{assump1}
  \item $\displaystyle \text{E}( \left\lVert \bm{X} \right\rVert_2) < \infty$\label{assump2}
  \item $\displaystyle (m_0, \bm{v}_0) \coloneqq \argsup_{m\in\mathbb{R},\ \bm{v}\in\mathcal{S}^{d-1}} f_{\bm{v}^{\top}\bm{X}}(m)$.\\
        $\displaystyle \left|m_0\right| < \infty$. $M\coloneqq [-m_0,m_0]$.\label{assump3}
  \item $\displaystyle L_0 \coloneqq \sup_{m\in\mathbb{R},\ \bm{v}\in\mathcal{S}^{d-1}} f_{\bm{v}^{\top}\bm{X}}(m) < \infty$\label{assump4}
  \item \label{assump5} $^{\forall}\bm{v}\in\mathcal{S}^{d-1},\ f_{\bm{v}^{\top}\bm{X}}(u)$ is a differentiable function with respect to $u$. 
\[
\displaystyle c_4 \coloneqq \sup_{u\in\mathbb{R},\ \bm{v}\in\mathcal{S}^{d-1}} \left\lvert \frac{df_{\bm{v}^{\top}\bm{X}}(u)}{du} \right\rvert < \infty
\]
  \item $\phi(z)$ is a differentiable kernel function with respect to $z$.\label{assump6}
  \item $\displaystyle c_0 \coloneqq \sup_{u\in\mathbb{R}} \left\lvert \phi(u) \right\rvert < \infty,\ 
        c_1 \coloneqq \sup_{u\in\mathbb{R}} \left\lvert \frac{d\phi(u)}{du} \right\rvert < \infty,\\
        c_2 \coloneqq \int \phi(z)^2 dz < \infty,\ 
        c_3 \coloneqq \int \left|u\right|\phi(u)du < \infty$ \label{assump7}
  \item $\left\{h_n\right\}_{n\in\mathbb{N}}$ is a positive bandwidth sequence such that\\
        $\displaystyle \lim_{n\to\infty}h_n = 0,\quad \lim_{n\to\infty}\frac{n h_n}{\log n} = \infty$. \label{assump8}
\end{enumerate}

\textbf{Proof.} Because the set $M\times \mathcal{S}^{d-1}$ is compact, we can find a finite set $J_n \subset M\times \mathcal{S}^{d-1}$ and a function $\lambda_n: M\times \mathcal{S}^{d-1} \ni (m,\bm{v}) \mapsto \lambda_n(m,\bm{v}) = \left( \lambda_n^{(1)}\left(m,\bm{v}\right),\ \lambda_n^{(2)}\left(m,\bm{v}\right) \right) \in J_n$, which satisfy the following properties:
\begin{align}
  &0<^{\exists}L<\infty,\ \left| J_n \right| \leq Ln^{2(d+1)}
,\\
  &^{\forall}(m,\bm{v})\in M\times \mathcal{S}^{d-1},\ \left\| \lambda_n\left(m,\bm{v}\right) - \left(m,\bm{v} \right) \right\|_2 \leq n^{-2}
.
\end{align}
We use $R_n\left( m,\bm{v}\right) \coloneqq \frac{1}{n}\sum_{i=1}^{n}\phi_h(m-\bm{v}^{\top}\bm{X}_{i})$ and $R\left(m,\bm{v}\right) \coloneqq f_{\bm{v}^{\top}\bm{X}}(m)$ and when there is no fear of confusion, we write $\lambda_n^{(1)} = \lambda_n^{(1)}(m,\bm{v}),\ \lambda_n^{(2)} = \lambda_n^{(2)}(m,\bm{v})$. 

Now, we obtain the following inequality:
\begin{align}
\notag
\left| R_n\left(m,\bm{v}\right) - R\left( m,\bm{v} \right) \right|
  \leq & \left| R_n\left(m,\bm{v}\right) - R_n\left( \lambda_n\left( m,\bm{v} \right) \right) \right| \notag \\
  &
        + \left| R_n\left( \lambda_n\left(m,\bm{v}\right) \right) - \text{E}\left[ R_n\left( \lambda_n\left(m,\bm{v}\right) \right) \right] \right| \notag \\ 
        &  + \left| \text{E}\left[ R_n\left( \lambda_n\left( m,\bm{v} \right) \right) \right] 
         - \text{E}\left[ R_n\left( m,\bm{v} \right) \right] \right| \notag \\
        &+ \left| \text{E}\left[ R_n\left( m,\bm{v} \right) \right] - R\left( m,\bm{v} \right) \right|.
\label{append:thoe4.1_target}
\end{align}
Convergence of each supremum term in probability is a sufficient condition for this proof.

First, we prove the following property:
\begin{align*}
  \sup_{(m,\bm{v})\in M\times \mathcal{S}^{d-1}} \left| R_n\left(m,\bm{v}\right) - R_n\left( \lambda_n\left( m,\bm{v} \right) \right) \right| = o_p(1).
\end{align*}
We have
\begin{align*}
  &\left| R_n\left(m,\bm{v}\right) - R_n\left( \lambda_n\left( m,\bm{v} \right) \right) \right|
\leq \frac{1}{nh_n}\sum_{i=1}^{n} \left| \phi\left( \frac{m-\bm{v}^{\top}\bm{X}_i}{h_n} \right) - \phi\left( \frac{\lambda_n^{(1)} -\lambda_n^{(2)\top}\bm{X}_i }{h_n} \right) \right|
\\
  &\text{by mean-value theorem and \ref{assump7}}
\\
  &\leq \frac{c_1}{(h_n)^2} \frac{1}{n} \sum_{i=1}^{n} \left| m-\lambda_n^{(1)} - \left( \bm{v}-\lambda_n^{(2)} \right)^{\top}\bm{X}_i \right|
\leq \frac{c_1}{(h_n)^2} \frac{1}{n} \sum_{i=1}^{n} \left| m-\lambda_n^{(1)} \right| + \left| \left( \bm{v}-\lambda_n^{(2)} \right)^{\top}\bm{X}_i \right|
\\
  &\leq \frac{c_1\sqrt{2}}{(h_n)^2} \frac{1}{n} \left\| (m,\bm{v})-\lambda_n(m,\bm{v}) \right\|_2 \sum_{i=1}^{n} \max\left( 1,\left\| \bm{X}_i \right\|_2 \right)
\\
  &\leq \frac{c_1\sqrt{2}}{(h_n)^2} \frac{1}{n} \left\| (m,\bm{v})-\lambda_n(m,\bm{v}) \right\|_2 \sum_{i=1}^{n} \left( 1 + \left\| \bm{X}_i \right\|_2 \right)
\\
  &\quad= \frac{c_1\sqrt{2}}{(h_n)^2}\left\| (m,\bm{v})-\lambda_n(m,\bm{v}) \right\|_2
             + \frac{c_1\sqrt{2}}{(h_n)^2}\left\| (m,\bm{v})-\lambda_n(m,\bm{v}) \right\|_2 \left[ \frac{1}{n}\sum_{i=1}^{n} \left\| \bm{X}_i \right\|_2 \right]
\\
  &\therefore \sup \left| R_n(m,\bm{v})-R\left( \lambda_n\left( m,\bm{v} \right)\right) \right|
    \; \leq \frac{c_1\sqrt{2}}{(nh_n)^2}
             + \frac{c_1\sqrt{2}}{(nh_n)^2} \left[ \frac{1}{n}\sum_{i=1}^{n} \left\| \bm{X}_i \right\|_2 \right].
\end{align*}
From Markov's inequality, for all $\epsilon>0$, we have
\begin{align*}
\text{P}\left( \left| \frac{c_1\sqrt{2}}{(nh_n)^2} + \frac{c_1\sqrt{2}}{(nh_n)^2} \left[ \frac{1}{n}\sum_{i=1}^{n} \left\| \bm{X}_i \right\|_2 \right] \right| \geq \epsilon \right)
 \leq \frac{\frac{c_1\sqrt{2}}{(nh_n)^2} + \frac{c_1\sqrt{2}}{(nh_n)^2} \text{E}\left[ \left\| \bm{X}_i \right\|_2 \right]}{\epsilon}.
\end{align*}  
Then, we apply \ref{assump2} and \ref{assump8} to obtain
\begin{align*}
\lim_{n\to\infty}\text{P}\left( \left| \frac{c_1\sqrt{2}}{(nh_n)^2} + \frac{c_1\sqrt{2}}{(nh_n)^2} \left[ \frac{1}{n}\sum_{i=1}^{n} \left\| \bm{X}_i \right\|_2 \right] \right| \geq \epsilon \right) = 0
\end{align*}
for all $\epsilon >0$. 
This means that $\frac{c_1\sqrt{2}}{(nh_n)^2} + \frac{c_1\sqrt{2}}{(nh_n)^2} \left[ \frac{1}{n}\sum_{i=1}^{n} \left\| \bm{X}_i \right\|_2 \right] = o_p(1)$, which proves $\sup \left| R_n(m,\bm{v}) - R\left( \lambda_n\left(m,\bm{v} \right)\right) \right| = o_p(1)$.

Second, we prove the convergence of the second term of Eq.~\eqref{append:thoe4.1_target}. For simplicity, we use the following notation:
\begin{align*}
a_i\left( \lambda_n(m,\bm{v}) \right)
\coloneqq \phi_{h_n}\left( \lambda_n^{(1)} - \lambda_n^{(2)\top}\bm{X}_i \right) - \text{E}\left[ \phi_{h_n}\left( \lambda_n^{(1)} - \lambda_n^{(2)\top}\bm{X}_i \right) \right]
\end{align*}
This makes the second term simple.
\begin{align*}
\left| R_n\left( \lambda_n\left( m,\bm{v} \right) \right) - \text{E}\left[ R_n\left( \lambda_n\left( m,\bm{v} \right) \right) \right] \right|
= \left| \frac{1}{n}\sum_{i=1}^{n} a_i\left( \lambda_n(m,\bm{v}) \right) \right|.
\end{align*}
In order to apply Bernstein inequalities, it is necessary to reveal the upper bound, lower bound, mean and variance boundedness of $a_i (m,\bm{v} )$. The mean is equal to $0$ obviously.
Bounds are derived as follows:
\begin{align*}
  -\frac{c_0}{h_n} \leq a_i\left( m,\bm{v} \right) \leq \frac{2c_0}{h_n}
.
\end{align*}
The variance is bounded as follows:
\begin{align*}
\text{Var}\left[ a_i\left( m,\bm{v} \right) \right]
\leq \frac{1}{(h_n)^2} \text{E}\left[ \phi\left( \frac{m-\bm{v}^{\top}\bm{X}}{h_n} \right)^2 \right]
\quad= \frac{1}{(h_n)^2} \int \phi\left( \frac{m-\bm{v}^{\top}\bm{x}}{h_n} \right)^2 dP_{\bm{X}}(\bm{x}),
\end{align*}
where $P_{\bm{X}}$ is the measure of $\bm{X}$. Rewriting $g(y)\coloneqq \phi\left( \frac{m-y}{h_n} \right)^2$ and $Y(\bm{x}) \coloneqq \bm{v}^{\top} \bm{x}$, we have
\begin{align*}
\text{Var}\left[ a_i\left( m,\bm{v} \right) \right]    
 \leq &\frac{1}{(h_n)^2} \int \left( g\circ Y \right)(\bm{x}) dP_{\bm{X}}(\bm{x})
\\
  &= \frac{1}{(h_n)^2} \int g(y) dP_{\bm{v}^{\top}\bm{X}}(y)
= \frac{1}{(h_n)^2} \int g(y) f_{\bm{v}^{\top}\bm{X}}(y) dy
\\
  &= \frac{1}{h_n} \int \phi(z)^2 f_{\bm{v}^{\top}\bm{X}}(m - zh_n) dz
\leq \frac{c_2 L_0}{h_n}.
\end{align*}
We change the variable as $z=\frac{m-y}{h_n}$ and apply \ref{assump4} and \ref{assump7} in the last line. 
From assumption~\ref{assump1}, $\{ a_n ( m,\bm{v}) \}_{n\in\mathbb{N}}$ are i.i.d. Hence we use the notation $a(m,\bm{v})$ as a random variable that follows a certain distribution. The Bernstein inequality leads to
\begin{align}
  ^{\forall}\epsilon>0,\ \text{P}\left(\left| \frac{1}{n}\sum_{i=1}^{n} a_i( m,\bm{v} ) \right| \geq \epsilon \right)
\notag
  &< 2\exp\left\{ - \frac{n \epsilon^2}{ 2\left( \text{Var}\left[ a(m,\bm{v}) \right] + \frac{3c_0 \epsilon}{h_n} \right) } \right\}
\notag\\
  &\leq 2\exp\left\{ - \frac{ \left(n h_n \right) \epsilon^2 }{ 2\left( c_2 L_0 + 3c_0 \epsilon \right) } \right\}.
\label{eq:append_second_term}
\end{align}
The above implies that $^{\forall}\epsilon>0$,
\begin{align*}
  &\text{P}\left( \sup_{(m,\bm{v})\in M\times \mathcal{S}^{d-1}} \left| \frac{1}{n}\sum_{i=1}^{n} a_i\left( \lambda_n(m,\bm{v}) \right) \right| \geq \epsilon \right)
\\
  &= \text{P}\left( \max_{(m,\bm{v})\in J_n} \left| \frac{1}{n}\sum_{i=1}^{n} a_i(m,\bm{v}) \right| \geq \epsilon \right)
  \leq \text{P}\left( \bigcup_{(m,\bm{v})\in J_n} \left\{ \left| \frac{1}{n}\sum_{i=1}^{n} a_i(m,\bm{v}) \right| \geq \epsilon \right\} \right)
\\
  &\text{by subadditivity of probability measures,}
\\
  &\leq \sum_{(m,\bm{v})\in J_n} \text{P}\left( \left| \frac{1}{n}\sum_{i=1}^{n}a_i(m,\bm{v}) \right| \geq \epsilon \right).
\\
  &\text{The set $J_n$ is finite and inequality~\eqref{eq:append_second_term} imply:}
\\
  &\leq 2\exp\left\{ - \frac{ \left(n h_n \right) \epsilon^2 }{ 2\left( c_2 L_0 + 3c_0 \epsilon \right) } \right\} Ln^{2(d+1)}
\\
  &= 2L\exp\left\{ -(n h_n) \left[ \frac{\epsilon^2}{2(c_2L_0 + 2c_0 \epsilon)} - 2(d+1)\frac{\log n}{n h_n} \right] \right\},
  \\
   &\text{and from assumption~\ref{assump8}, we obtain},
\\
 &\lim_{n\to\infty} \text{P}\left( \sup_{(m,\bm{v})\in M\times \mathcal{S}^{d-1}} \left| \frac{1}{n}\sum_{i=1}^{n} a_i\left( \lambda_n(m,\bm{v}) \right) \right| \geq \epsilon \right) = 0,
\end{align*}
which proves $\sup \left| R_n\left( \lambda_n\left( m,\bm{v} \right) \right) - \text{E}\left[ R_n\left( \lambda_n\left( m,\bm{v} \right) \right) \right] \right| = o_p(1)$.

Next, we show the convergence of the third term of Eq.~\eqref{append:thoe4.1_target}.
\begin{align*}
  &\left| \text{E}\left[ R_n\left( \lambda_n(m,\bm{v}) \right) \right] - \text{E}\left[ R_n(m,\bm{v}) \right] \right|
\\
  &\leq \frac{1}{h_n} \int \left| \phi\left( \frac{\lambda_n^{(1)} - \lambda_n^{(2)\top}\bm{x}}{h_n} \right) - \phi\left( \frac{m-\bm{v}^{\top}\bm{x}}{h_n} \right) \right| dP_X(\bm{x})
\\
  &\text{by mean-value theorem and \ref{assump7}}
\\
  &\leq \frac{c_1}{(h_n)^2} \int \left\{ \left| \lambda_n^{(1)} - m \right| + \left| (\lambda_n^{(2)} - \bm{v})^{\top}\bm{x} \right| \right\} dP_X(\bm{x})
\\
  &\leq \frac{c_1 \sqrt{2}}{(h_n)^2} \left\| \lambda_n(m,\bm{v}) - (m,\bm{v}) \right\|_2 \int \max\left(1, \left\| \bm{x} \right\|_2\right) dP_X(\bm{x})
\\
  &\leq \frac{c_1 \sqrt{2}}{(h_n)^2} \left\| \lambda_n(m,\bm{v}) - (m,\bm{v}) \right\|_2
        + \frac{c_1}{(h_n)^2} \left\| \lambda_n(m,\bm{v}) - (m,\bm{v}) \right\|_2 \text{E}\left[ \left\| X \right\|_2 \right]
\\
  &\therefore \sup_{(m,\bm{v})\in M\times \mathcal{S}^{d-1}} \left| \text{E}\left[ R_n\left( \lambda_n(m,\bm{v}) \right) \right] - \text{E}\left[ R_n(m,\bm{v}) \right] \right|
\\
  &\leq \frac{c_1 \sqrt{2}}{(n h_n)^2} \left( 1+\text{E}\left[ \left\| \bm{X} \right\|_2 \right] \right)
\end{align*}
We note that $\frac{c_1 \sqrt{2}}{(n h_n)^2} \left( 1+\text{E}\left[ \left\| \bm{X} \right\|_2 \right] \right)$ depends only on $n$. From Markov's inequality, $^{\forall}\epsilon > 0$,
\begin{align*}
  &\text{P}\left( \left| \frac{c_1 \sqrt{2}}{(n h_n)^2} \left( 1+\text{E}\left[ \left\| \bm{X} \right\|_2 \right] \right) \right| \geq \epsilon \right) \leq \frac{c_1\sqrt{2}\left( 1+\text{E}\left[ \left\| \bm{X} \right\|_2 \right] \right)}{\epsilon(n h_n)^2}
\\
  &\lim_{n\to\infty} \text{P}\left( \left| \frac{c_1 \sqrt{2}}{(n h_n)^2} \left( 1+\text{E}\left[ \left\| \bm{X} \right\|_2 \right] \right) \right| \geq \epsilon \right) = 0,
\end{align*}
which proves that $\sup \left| \text{E}\left[ R_n\left( \lambda_n(m,\bm{v}) \right) \right] - \text{E}\left[ R_n(m,\bm{v}) \right] \right| = o_p(1)$.

Finally, we provide the proof of the convergence of the fourth term of Eq.~\eqref{append:thoe4.1_target}:
\begin{align*}
  &\left| \text{E}\left[ R_n\left( m,\bm{v} \right) \right] - R\left( m,\bm{v} \right) \right|
  = \left| \int \phi_{h_n}(m-\bm{v}^{\top} \bm{x}) dP_{\bm{X}}(\bm{x}) - f_{\bm{v}^{\top}\bm{X}}(m) \right|
\\
  &= \left| \int \phi_{h_n}(m-y)dP_{\bm{v}^{\top}\bm{X}}(y) - f_{\bm{v}^{\top}\bm{X}}(m) \right|
\\
  &= \left| \frac{1}{h_n}\int \phi\left(\frac{m-y}{h_n}\right) f_{\bm{v}^{\top}\bm{X}}(y)dy - f_{\bm{v}^{\top}\bm{X}}(m) \right|
\\
  &\text{setting $z=\frac{m-y}{h_n}$, it follows that}
\\
  &= \left| \int \phi(z) f_{\bm{v}^{\top}X}(m-z h_n) dz - f_{\bm{v}^{\top}X}(m) \right|
\\
  &\text{where $\int \phi(z)dz=1$ and hence that}
\\
  &\leq \int \phi(z) \left| f_{\bm{v}^{\top}\bm{X}}(m-z h_n) - f_{\bm{v}^{\top}\bm{X}}(m) \right| dz
\\
  &\text{by mean-value thoerem and \ref{assump5}, \ref{assump7}}
\\
  &\leq h_n c_3 c_4
\\
  &\therefore \sup_{(m,\bm{v})\in M\times \mathcal{S}^{d-1}} \left| \text{E}\left[ R_n\left( m,\bm{v} \right) \right] - R\left( m,\bm{v} \right) \right| \leq h_n c_3 c_4
\end{align*}
Assumption~\ref{assump8} leads to $\lim_{n\to\infty} h_n c_3 c_4 = 0$. Hence $\sup \left| \text{E}\left[ R_n\left( m,\bm{v} \right) \right] - R\left( m,\bm{v} \right) \right| = o_p(1)$ holds.

Putting these results together, we obtain $\sup \left| R_n\left(m,\bm{v}\right) - R\left( m,\bm{v} \right) \right| = o_p(1)$.

\section*{Appendix B: Proof of Theorem~\ref{theo:convergencerate}}

We show that Theorem~\ref{theo:convergencerate} holds. Let us denote
\begin{align*}
  h_n &\coloneqq
    M_0 n^{-\frac{1}{k}}, \quad\text{where}\quad M_0>0,\quad k>4
,\\
  \delta_n^{(1)} &\coloneqq
    \sup_{(m,\bm{v})\in M\times \mathcal{S}^{d-1}} \left| R_n\left( m,\bm{v} \right) - R_n\left( \lambda_n(m,\bm{v}) \right) \right|
,\\
  \delta_n^{(2)} &\coloneqq
    \sup_{(m,\bm{v})\in M\times \mathcal{S}^{d-1}} \left| R_n\left( \lambda_n(m,\bm{v}) \right) - \text{E}\left[ R_n\left( \lambda_n(m,\bm{v}) \right) \right] \right|
,\\
  \delta_n^{(3)} &\coloneqq
    \sup_{(m,\bm{v})\in M\times \mathcal{S}^{d-1}} \left| \text{E}\left[ R_n\left( \lambda_n(m,\bm{v}) \right) \right] - \text{E}\left[ R_n\left( m,\bm{v} \right) \right] \right|
,\\
  \delta_n^{(4)} &\coloneqq
    \sup_{(m,\bm{v})\in M\times \mathcal{S}^{d-1}} \left| \text{E}\left[ R_n\left( m,\bm{v} \right) \right] - R\left( m,\bm{v} \right) \right|
.
\end{align*}

In Appendix~A, we see that
\begin{align*}
  &\delta_n^{(1)} \leq
    (n h_n)^{-2} c_1 \sqrt{2} \left[ 1+\frac{1}{n}\sum_{i=1}^{n} \left\| \bm{X}_i \right\|_2 \right]
,\\
  &^{\forall}\delta > 0, \quad
    P\left( (n h_n)^{-2} c_1 \sqrt{2} \left[ 1+\frac{1}{n}\sum_{i=1}^{n} \left\| \bm{X}_i \right\|_2 \right] > \epsilon \right) \leq
      \epsilon^{-1} (n h_n)^{-2} c_1 \sqrt{2} \left\{ 1+\text{E}\left[ \left\| \bm{X} \right\|_2 \right] \right\}
\end{align*}
is satisfied. Let $\epsilon = c_1 \sqrt{2} \left\{ 1+\text{E}\left[ \left\| \bm{X} \right\|_2 \right] \right\} h_n$. Then,
\begin{align}
  \sum_{n\in\mathbb{N}} P(\delta_n^{(1)} > M_1 n^{-\frac{1}{k}} ) \leq
    M_0^{-3} \sum_{n\in\mathbb{N}} \frac{1}{n^{2-\frac{3}{k}}} < \infty
,\label{appendB:bound1}
\end{align}
where $M_1 \coloneqq c_1 \sqrt{2} \left\{ 1+\text{E}\left[ \left\|\bm{X}\right\|_2 \right] \right\} M_0$.

The quantity $\frac{1}{n}\phi_{h_n}\left( \lambda_n^{(1)} - \lambda_n^{(2)\top} \bm{X} \right)$ is bounded as
\begin{align*}
  0 \leq
  \frac{1}{n}\phi_{h_n}\left( \lambda_n^{(1)} - \lambda_n^{(2)\top} \bm{X} \right) \leq
  (n h_n)^{-1} c_0
.
\end{align*}
Therefore, by Hoeffding's inequality, for any $\epsilon > 0$,
\begin{align*}
  &P\left( \left| \frac{1}{n}\sum_{i=1}^{n} \phi_{h_n}\left( \lambda_n^{(1)} - \lambda_n^{(2)\top}\bm{X} \right) - \text{E}\left[ \phi_{h_n}\left( \lambda_n^{(1)} - \lambda_n^{(2)\top}\bm{X} \right) \right] \right| \geq \epsilon \right)
    \leq 2\exp\left\{ -2c_0^{-2} \epsilon^2 n h_n^2 \right\}
,\\
  &\therefore
  P\left( \delta_n^{(2)} \geq \epsilon \right) \leq
    2L n^{2(d+1)} \exp\left\{ -2c_0^{-2} \epsilon^2 n h_n^2 \right\}, \quad \text{let $\epsilon = \frac{c_0}{M\sqrt{2}} h_n$}
,\\
  &\therefore
  \sum_{n\in\mathbb{N}} P\left( \delta_n^{(2)} \geq M_2 n^{-\frac{1}{k}} \right) \leq
    2L \sum_{n\in\mathbb{N}} n^{2(d+1)} \exp\left\{ -n^{1-\frac{4}{k}} \right\}, \quad \text{where}\quad M_2 \coloneqq \frac{c_0}{\sqrt{2}}
.
\end{align*}
Let us consider $g(z) \coloneqq z^{2(d+1)} \exp\left\{ -z^{1-\frac{4}{k}} \right\}$. For any $z \geq z_0 \coloneqq \left\{ \frac{2(d+1)}{1-\frac{4}{k}} \right\}^{\frac{1}{1-\frac{4}{k}}}$, $g(z)$ is monotonically non-increasing. In addition,
\begin{align*}
  \int_{z_0}^{\infty} g(z) dz = \frac{k}{k-4} \Gamma\left( \frac{k}{k-4}(2d+3), z_0^{1-\frac{4}{k}} \right)
,
\end{align*}
where $\Gamma(\alpha,\beta) \coloneqq \int_{\beta}^{\infty} z^{\alpha-1} e^{-z} dz$ denotes the upper incomplete gamma function.
Now we see that the infinite series $\sum_{n\in\mathbb{N}} g(n)$ is bounded by the integral test, hence
\begin{align}
  \sum_{n\in\mathbb{N}} P\left( \delta_n^{(2)} \geq M_2 n^{-\frac{1}{k}} \right) < \infty
.\label{appendB:bound2}
\end{align}

Recalling that the inequality $\delta_n^{(3)} \leq (n h_n)^{-2} c_1 \sqrt{2} \left\{ 1+\text{E}\left[ \left\| \bm{X} \right\|_2 \right] \right\}$ is satisfied, we obtain $\delta_n^{(3)} = o(h_n)$ and
\begin{align}
  \sum_{n\in\mathbb{N}}P\left( \delta_n^{(3)} > M_3 n^{-\frac{1}{k}} \right) < \infty, \quad\text{where}\quad M_3 \coloneqq M_0
.\label{appendB:bound3}
\end{align}

The forth team~$\delta_n^{(4)}$ satisfies with $\delta_n^{(4)} \leq c_3 c_4 h_n$, which gives
\begin{align}
  \sum_{n\in\mathbb{N}}P\left( \delta_n^{(4)} > M_4 n^{-\frac{1}{k}} \right) < \infty,\quad \text{where}\quad M_4 \coloneqq c_3 c_4
.\label{appendB:bound4}
\end{align}

Finally, let $M \coloneqq 4 \max\left( M_1,M_2,M_3,M_4 \right)$. Then,
\begin{align*}
  P\left( \sup \left| R_n\left( m,\bm{v} \right) - R(m,\bm{v}) \right| > M n^{-\frac{1}{k}} \right) &\leq
    \sum_{j=1}^{4} P\left( \delta_n^{(j)} > \frac{M}{4} n^{-\frac{1}{k}} \right)
\\
    &\leq \sum_{j=1}^{4} P\left( \delta_n^{(j)} > M_j n^{-\frac{1}{k}} \right)
.
\end{align*}
The inequalities~\eqref{appendB:bound1},~\eqref{appendB:bound2},~\eqref{appendB:bound3},~\eqref{appendB:bound4} leads that $\sum_{n\in\mathbb{N}} P\left( \sup \left| R_n\left( m,\bm{v} \right) - R(m,\bm{v}) \right| > M n^{-\frac{1}{k}} \right)$ is bounded, which means $\sup \left| R_n\left(m,\bm{v}\right) - R(m,\bm{v}) \right| = \mathcal{O}_{\text{a.co.}}(n^{-\frac{1}{k}})$.

\section*{Appendix C: Proof of Theorem~\ref{theo:IF}}
We derive the influence function of $\hat{\bm{w}}_k$ in Theorem~\ref{theo:IF}. In this section, $\hat{\bm{w}}_k(F_{\bm{Y}})$ denotes the estimator which is an optimal solution of problem~(4) given the probability measure $F_{\bm{Y}}$.

For $k \geq 2$, the optimization problem~(4) results in the Lagrangian function
\begin{align*}
  \mathcal{L}(\bm{w}, \gamma, \alpha_1,\dots ,\alpha_{k-1})
= \int \phi_h(\bm{w}^{\top} \bm{y})dF_{\bm{Y}}(\bm{y}) + \gamma (1-\bm{w}^{\top}\bm{w}) + \sum_{l=1}^{k-1} \alpha_l \bm{w}^{\top}\hat{\bm{w}}_{l}.
\end{align*}
The estimator $\hat{\bm{w}}_{k}$ maximizes the Lagrangian function and needs to satisfy
\begin{align*}
  &\frac{\partial \mathcal{L}}{\partial \bm{w}} = \bm{0}
,\quad
  \frac{\partial \mathcal{L}}{\partial \gamma} = 0
,\quad
  \frac{\partial \mathcal{L}}{\partial \alpha_l} = 0
,\quad
  l=1\dots k-1
.\\
&\Rightarrow\ \left\{\begin{aligned}
  &\alpha_l = -\hat{\bm{w}}_l(F_{\bm{Y}})^{\top}\bm{\psi}(\hat{\bm{w}}_k(F_{\bm{Y}}), F_{\bm{Y}})
,\quad
  l=1\dots k-1
,\\
  &2\gamma = \hat{\bm{w}}_{k}(F_{\bm{Y}})^{\top}\bm{\psi}(\hat{\bm{w}}_k(F_{\bm{Y}}), F_{\bm{Y}})
.
\end{aligned}\right.
\end{align*}
The Lagrangian function satisfies
\begin{align}
\left.\frac{\partial \mathcal{L}}{\partial \bm{w}} \right|_{\hat{\bm{w}}_k(F_{\bm{Y}})}
= \left[ I - \sum_{l=1}^{k}\hat{\bm{w}}_l(F_{\bm{Y}})\hat{\bm{w}}_l(F_{\bm{Y}})^{\top} \right] \bm{\psi}(\hat{\bm{w}}_k(F_{\bm{Y}}), F_{\bm{Y}})
   = \bm{0}
\label{eq:append_lagrangian}.
\end{align}
Suppose that the probability measure is $F_{\bm{Y}} = (1-\epsilon) F + \epsilon \Delta_u$, where $F, \Delta_{\bm{u}}$ denote the true probability measure and the Dirac measure at $\bm{u}\in\mathbb{R}^d$, respectively. Then, the influence function $\text{IF}(\bm{u};\hat{\bm{w}}_k, F)$ is expressed as $\left.\frac{\partial}{\partial \epsilon} \hat{\bm{w}}_k(F_Y) \right|_{\epsilon=0}$. Differentiation of each term in \eqref{eq:append_lagrangian} yields
\begin{align*}
  &\left.\frac{d}{d\epsilon} \bm{\psi}(\hat{\bm{w}}_k(F_{\bm{Y}}), F_{\bm{Y}})\right|_{0}
   = -\bm{d}_{k} + \bm{B}_{k} \left.\frac{d}{d\epsilon} \hat{\bm{w}}_k(F_{\bm{Y}}) \right|_{0}
,\\
  &\frac{d}{d\epsilon} \left\{ I - \sum_{l=1}^{k} \hat{\bm{w}}_l(F_{\bm{Y}}) \hat{\bm{w}}_l(F_{\bm{Y}})^{\top} \right\}_{0}
= -\sum_{l=1}^{k} \left[ \left. \frac{d}{d\epsilon}\hat{\bm{w}}_l(F_{\bm{Y}}) \right|_{0} \hat{\bm{w}}_l(F)^{\top} + \hat{\bm{w}}_l(F) \left. \frac{d}{d\epsilon}\hat{\bm{w}}_l(F_Y) \right|_{0}^{\top} \right]
.
\end{align*}
 respectively.
Therefore, the derivative of \eqref{eq:append_lagrangian} yields
\begin{align*}
  &\left.\frac{d}{d\epsilon} \left\{ \left[ I - \sum_{l=1}^{k}\hat{\bm{w}}_l(F_{\bm{Y}})\hat{\bm{w}}_l(F_{\bm{Y}})^{\top} \right] \bm{\psi}\left( \hat{\bm{w}}_k(F_{\bm{Y}}), F_{\bm{Y}} \right) \right\}\right|_{0} = \bm{0}
,
\end{align*}
\begin{align*}
  \Rightarrow \left. \frac{d}{d\epsilon} \hat{\bm{w}}_k(F_{\bm{Y}}) \right|_{0}
   &= \left( \bm{A}_k \bm{B}_k - \bm{C}_k \right)^{-1}
 \times \left[ \bm{A}_k \bm{d}_k - \sum_{l=1}^{k-1} \bm{C}_l \left.\frac{d}{d\epsilon} \hat{\bm{w}}_l(F_{\bm{Y}}) \right|_{0} \right]
.
\end{align*}
The case of $k=1$ is derived in the same way.

\section*{Appendix D: Proof of Theorem~\ref{theo:inequality}}
This section is devoted to the proof of an LBBP in Theorem~\ref{theo:inequality}. Recall that $\phi(z)=\exp\left(-z^2/2\right)/\sqrt{2\pi}$.

\textbf{Proof}\mbox{}\\ 
From
\begin{align*}
  &h\sqrt{2\pi}\phi_h(z) = \exp\left\{ -\frac{1}{2}\left( \frac{z}{h} \right)^2 \right\}
  \leq e^0 = 1,
\end{align*}
the inequality $^{\forall}z\in\mathbb{R},\ h\sqrt{2\pi}\phi_h(z) \leq 1$ holds. 
Next, Theorem~4.3 requires us to prove that
\begin{align*}
  ^{\forall}b\in\mathbb{N},\quad
    & b < M_a(\hat{\bm{w}}_1(Y_a)) - M_a^{*}(\hat{\bm{w}}_1(Y_a))
\\
    &\quad \Rightarrow ^{\forall}Y_b\subset \mathbb{R}^d,\ \hat{\bm{w}}_1(Y_a\cup Y_b)^{\top} \hat{\bm{w}}_1(Y_a) \neq 0
.
\end{align*}
To derive the contradiction, suppose that $^{\exists}m_0\in\mathbb{N}$ where $m_0 < M_a(\hat{\bm{w}}_1(Y_a)) - M_a^{*}(\hat{\bm{w}}_1(Y_a)),\; ^{\exists}Y_{m_0} = \left\{ \bm{x}_{a+i} \right\}_{i=1}^{m_0} \subset \mathbb{R}^d,\; \hat{\bm{w}}_1(Y_a \cup Y_{m_0})^{\top}\hat{\bm{w}}_1(Y_a)=0$. This yields
\begin{align*}
  h\sqrt{2\pi}\sum_{i=1}^{a}\phi_h\left( \hat{\bm{w}}_1(Y_a\cup Y_{m_0})^{\top}\bm{x}_i \right) \leq M_a^{*}(\hat{\bm{w}}_1(Y_a))
.
\end{align*}
Hence, it yields
\begin{align*}
&h\sqrt{2\pi}\sum_{i=1}^{a+m_0} \phi_h\left( \hat{\bm{w}}_1(Y_a\cup Y_{m_0})^{\top}\bm{x}_i \right)
  \leq M_a^{*}(\hat{\bm{w}}_1(Y_a)) + m_0
< M_a(\hat{\bm{w}}_1(Y_a))\\
&= h\sqrt{2\pi}\sum_{i=1}^{a}\phi_h\left( \hat{\bm{w}}_1(Y_a)^{\top}\bm{x}_i \right)
\leq h\sqrt{2\pi}\sum_{i=1}^{a+m_0} \phi_h\left( \hat{\bm{w}}_1(Y_a)^{\top}\bm{x}_i \right)
\\
  &\therefore \sum_{i=1}^{a+m_0} \phi_h\left( \hat{\bm{w}}_1(Y_a\cup Y_{m_0})^{\top}\bm{x}_i \right)
   < \sum_{i=1}^{a+m_0} \phi_h\left( \hat{\bm{w}}_1(Y_a)^{\top}\bm{x}_i \right)
.
\end{align*}
This is a contradiction, because $\hat{\bm{w}}_1(Y_a\cup Y_{m_0})$ is an optimal solution that maximizes $\sum_{i=1}^{a+m_0} \phi_h\left( \bm{w}^{\top}\bm{x}_i \right)$.

%
%
%
\section*{Appendix E: Algorithmic Detail}
In this section, we describe the algorithmic details of the optimization for the proposed MPCA method. The objective function of problem~\eqref{eq:optim1} is 
 \begin{align*}
    \begin{aligned}
    &(\hat{m}_{k},\hat{\bm{v}}_{k})
      =  \argmax_{m\in\mathbb{R},\ \bm{v}\in \mathcal{S}^{d-1} }\quad \frac{1}{N}\sum_{i=1}^{N}\phi_h\left( m-\bm{v}^{\top}\bm{x}_{i} \right)
  ,\\ &
    \quad \text{s.t.}\quad
        \bm{v}^{\top}\hat{\bm{v}}_{j}=0, \quad j=1\dots k-1
      .
    \end{aligned}
\end{align*}
which non-convex with respect to $m,\bm{v}$ and hence, it is difficult to obtain the global optimum. We introduce the GRID algorithm~\citep{Croux2007} to obtain a good initial point, and then an algorithm that converges to a local optima given an initial point is proposed.

\subsection*{Selection of Initial Point}
Since the objective function in the problem~\eqref{eq:optim1} is non-convex, it is important to select a good initial solution. We adopt the GRID algorithm proposed in the projection-pursuit robust PCA method~\citep{Croux2007}. The GRID algorithm searches the unit sphere for a better direction by multi-step grid search-like approach. \citet{Croux2007} empirically shows that the GRID algorithm is able to evaluate most of possible directions efficiently.

We explain the GRID algorithm by an example of estimating the classical PCA projection direction. We introduce the notation$\text{Var}(\bm{a})=\text{Var}(\bm{a}^{\top}\bm{x}_{1}, \dots , \bm{a}^{\top}\bm{x}_{N})$. When $d=2$, the estimate $\hat{\bm{a}}_{1}$ of $\text{PC}_{1}$ is obtained by finding $\theta \in \left[ -\pi/2, \pi/2 \right)$, which maximizes $\text{Var} \left( 
( \cos(\theta)\; \sin(\theta) )^{\top} \right)$. The GRID algorithm evaluates the objective function $\text{Var}(\cdot)$ at $N_g$ grid points $\{ (-1/2+j/N_g)\pi \}_{j=0}^{N_g-1}$ in $\left[-\pi/2,\pi/2\right)$. By using the maximizer $\hat{\theta}_1$ for $\text{Var}(\cdot)$ among the grid points, the solution is updated to $( \cos(\hat{\theta}_1)\ \sin(\hat{\theta}_1) )^{\top}$. We then consider $N_g$ grid points $\{ \hat{\theta}_1+ (-1/2^2+j/2N_g )\pi \}_{j=0}^{N_g-1}$ on $\left[ \hat{\theta}_1-\pi/2^2, \hat{\theta}_1+\pi/2^2  \right)$ to find the maximizer $\hat{\theta}_2$, and update the solution by $(\cos(\hat{\theta}_2)\ \sin(\hat{\theta}_2) )^{\top}$. We iterate those procedures $N_c$ times to obtain the initial solution of the problem~\eqref{eq:optim1}. When $d>2$, we apply the same algorithm for $d=2$ on the two-dimensional subspaces spanned by current solution $\hat{\bm{a}}$ and a basis $\bm{e}_i,\; i=1\dots d$. Details of the algorithm are described in ~\citep{Croux2007}. 

\subsection*{Optimization Problem on a Manifold}
\label{sec:optimizationproblemonamanifold}
In this subsection, we propose an algorithm to find a local optimal solution of problem~\eqref{eq:optim1} given an initial point $\bm{v}^{(0)}$. Since the simultaneous optimization of $m$ and $\bm{v}$ is difficult, we solve the problem~\eqref{eq:optim1} in an iterative manner.

\subsubsection{Optimize Model with Fixed Projection Axes}
\label{sec:optimize_m_with_fixed_vl}
We update $m$ by solving the following unconstrained optimization problem with respect to $m$:
\begin{align}
  m^{(l)}
    &=\argmax_{m\in\mathbb{R}} \; \frac{1}{N}\sum_{i=1}^{N} \phi_h(m-\bm{v}^{(l)\top}\bm{x}_i)
.
\end{align}
The solution $m$ of the problem is regarded as an estimate of mode, and the half-sample mode method~\citep{Bickel2006} can be used as a fast and reasonable estimator for this problem. In this study, we use the estimate $m_{(0)}$ obtained by the half-sample mode method as an initial point and apply the Newton method to obtain a higher precision solution. The concrete update formula is given as follows:
\begin{align*}
  &m_{(j+1)}
    =  m_{(j)} - \frac{F(m_{(j)})}{\frac{dF(m_{(j)})}{dm}}
,\\
&\quad 
\left\{
  \begin{aligned}
    &F(m)=\sum_{i=1}^{N}(m-\bm{v}^{(l)\top}\bm{x}_{i}) \phi_h(m-\bm{v}^{(l)\top}\bm{x}_{i})
  ,\\
    &\frac{dF(m)}{dm}
    =\sum_{i=1}^{N} \left\{ 1-\left( \frac{m-\bm{v}^{(l)\top}\bm{x}_{i}}{h} \right)^2 \right\} \phi_h(m-\bm{v}^{(l)\top}\bm{x}_{i})
  .
  \end{aligned}
  \right.
\end{align*}

\subsubsection*{Optimize Projection Axis with Fixed Mode}
Consider the following optimization problem:
\begin{align}
  \begin{aligned}
    &\max_{\bm{v}\in\mathcal{S}^{d-1}} \quad \log \left[ \frac{1}{N} \sum_{i=1}^{N}\phi_h\left( m^{(l)}-\bm{v}^{\top}\bm{x}_{i} \right) \right]
\quad 
\text{s.t.}\quad 
      \bm{v}^{\top}\hat{\bm{v}}_{j}=0, \quad j=1\dots k-1.
  \end{aligned}
\end{align}
In this problem, $\frac{1}{N} \sum_{i=1}^{N}\phi_h\left( m^{(l)}-\bm{v}^{\top}\bm{x}_i \right)$ is non-negative and we can take the logarithm of the objective. By Jensen's inequality, we obtain
\begin{align}
  &\log \left[ \frac{1}{N}\sum_{i=1}^{N}\phi_h(m^{(l)}-\bm{v}^{\top}\bm{x}_{i}) \right]
  \geq \sum_{i=1}^{N} q_{i}^{(l)} \log \phi_h(m^{(l)}-\bm{v}^{\top}\bm{x}_{i})
        -\log N
        -\sum_{i=1}^{N}q_{i}^{(l)}\log q_{i}^{(l)},
\end{align}
where $q_{i}^{(l)}=\frac{\phi_h(m^{(l)}-\bm{v}^{(l)\top}\bm{x}_{i})}{\sum_{j=1}^{N}\phi_h(m^{(l)}-\bm{v}^{(l)\top}\bm{x}_{j})}$. Then, as a relaxation of the original problem~\eqref{eq:optimize_v}, we consider the following problem and update the estimate by its solution:
\begin{align}
  \begin{aligned}
    &\bm{v}^{(l+1)}
      =  \argmax_{\bm{v}\in\mathcal{S}^{d-1}} \; \sum_{i=1}^{N} q_{i}^{(l)} \log \phi_h(m^{(l)}-\bm{v}^{\top}\bm{x}_{i})
    \quad 
    \text{s.t.}\quad
        \bm{v}^{\top}\hat{\bm{v}}_{j}=0, \quad j=1\dots k-1.
  \end{aligned}
\end{align}
We note that $\phi_h(z)=\phi\left(z/h\right)/h,\ \phi(z)=\exp\left(-z^2/2\right)/\sqrt{2\pi}$ and hence, the above problem is equivalent to the following problem:
\begin{align}
  \begin{aligned}
    &\bm{v}^{(l+1)}
      =  \argmin_{\bm{v}\in\mathcal{S}^{d-1}} \quad \sum_{i=1}^{N} q_{i}^{(l)} (m^{(l)}-\bm{v}^{\top}\bm{x}_{i})^2
   \quad
   \text{s.t.}\quad
        \bm{v}^{\top}\hat{\bm{v}}_{j}=0, \quad j=1\dots k-1.
  \end{aligned}
\end{align}
We set $G^{(l)}(\bm{v})=\sum_{i=1}^{N} q_{i}^{(l)}(m^{(l)}-\bm{v}^{\top}\bm{x}_{i})^2$ henceforth.

Since the domain of $\bm{v}\in\mathbb{R}^d$ is restricted to $\bm{v}\in\mathcal{S}^{d-1}$, the constrained optimization problem~\eqref{eq:optimize_v3} is formulated as an unconstrained optimization problem on a manifold $\mathcal{S}^{d-1}$. There are many sophisticated methods for dealing with optimization problems on special manifolds~\citep{Absil2007}. In our problem, the manifold $\mathcal{S}^{d-1}$ is a simple set and we can explicitly calculate its local coordinate $\bm{\varphi}_{0}:\mathcal{S}^{d-1}\setminus\left\{-\bm{v}_{0}\right\}\to\mathbb{R}^{d-k}$ as follows:
\begin{enumerate}
  \item Set $\bm{v}_{0}\in\mathcal{S}^{d-1}$ so that $j=1\dots k-1,\ \bm{v}_{0}^{\top}\hat{\bm{v}}_{j}=0$.
  \item Set $\bm{U}=(\bm{u}_1\ \dots \ \bm{u}_{d-k})\in\mathbb{R}^{d\times (d-k)}$ so that $\bm{U}^{\top} \bm{U} = I,\; \bm{U}^{\top}\bm{v}_{0} = \bm{0},\; j=1\dots k-1,\; \bm{U}^{\top} \hat{\bm{v}}_j = \bm{0}$.
  \item Let $\bm{\varphi}_{0}:\mathcal{S}\setminus \left\{-\bm{v}_{0}\right\}\ni \bm{v}\mapsto \bm{\varphi}_{0}(\bm{v})=\frac{1}{1+\bm{v}_{0}^{\top}\bm{v}}\bm{U}^{\top}\bm{v} \in \mathbb{R}^{d-k} $, which leads to
   $\bm{\varphi}_{0}^{-1}:\mathbb{R}^{d-k}\ni \bm{\beta}\mapsto \bm{\varphi}_{0}^{-1}(\bm{\beta})=\frac{1}{1+\bm{\beta}^{\top}\bm{\beta}}(2\bm{U}\bm{\beta} + (1-\bm{\beta}^{\top}\bm{\beta})\bm{v}_{0})\in \mathcal{S}^{d-1}$.
\end{enumerate}
We used the local coordinate defined above to solve the following unconstrained problem:
\begin{align}
  \bm{\beta}^{(l+1)}
    =  \argmin_{\bm{\beta}\in\mathbb{R}^{d-k}} \quad G^{(l)}(\bm{\varphi}_{0}^{-1}(\bm{\beta}))
\tag{9}
\end{align}
is shown to be equivalent to solving~\eqref{eq:optimize_v3}.

The following lemma ensures that we can solve the unconstrained problem~\eqref{eq:optimize_v4} instead of the constrained problem~\eqref{eq:optimize_v3}. In practice, we estimate $\frac{\partial}{\partial \bm{\beta}}G(\bm{\varphi}_{0}^{-1}(\bm{\beta}))$ and we can use the gradient method or conjugate gradient method to obtain the solution.

\begin{lemm}
\label{lemm:homeomorphic}
  $M = \left\{ \bm{v}\in\mathcal{S}^{d-1} \middle| \begin{aligned} &l=1\dots k-1, \bm{v}^{\top} \hat{\bm{v}}_{l} = 0,\\ &v \neq -\bm{v}_{0} \end{aligned} \right\}$ is homeomorphic to $\mathbb{R}^{d-k}$, and $\varphi_{0}$ is the homeomorphism.
\end{lemm}

\begin{proof}

We provide the proof of Lemma~\ref{lemm:homeomorphic} by showing that (i) $\varphi_{0}$ is bijective and continuous and that (ii) $\varphi_{0}^{-1}$ is the inverse function of $\varphi_{0}$ and continuous.
We note that the following property shown in Eq.~\eqref{lemm:homeomorphic_prop1} is satisfied because $\left\{ \hat{\bm{v}}_{1}, \dots ,\hat{\bm{v}}_{k-1}, \bm{v}_{0}, \bm{u}_{1},\dots , \bm{u}_{d-k} \right\}$ is an orthonormal basis for $\mathbb{R}^{d}$:
\begin{align}
&  ^{\forall} \bm{v}\in M,\quad ^{\exists} \bm{\beta}_{1}\in\mathbb{R}^{d-k},\quad ^{\exists} \beta_{2}\in\mathbb{R}\setminus \left\{ -1 \right\},\\
&
    \left\| \bm{\beta}_{1} \right\|_{2}^{2} + \beta_{2}^{2} = 1
  \quad \land \quad
    \bm{v}=\bm{U}\bm{\beta}_{1} + \beta_{2}\bm{v}_{0}
\label{lemm:homeomorphic_prop1}
.
\end{align}

First, we show that $\varphi_{0}$ is injective by proving $^{\forall} \bm{v},\bm{w}\in M,\; \varphi_{0}(\bm{v}) = \varphi_{0}(\bm{w}) \Rightarrow \bm{v} = \bm{w}$.
The property~\eqref{lemm:homeomorphic_prop1} implies that there exist $\bm{\alpha}_{1}, \bm{\beta}_{1}\in\mathbb{R}^{d-k},\; \alpha_{2},\; \beta_{2}\in\mathbb{R}$ such that $\bm{v}=\bm{U}\bm{\alpha}_{1} + \alpha_{2} \bm{v}_{0},\; \bm{w}=\bm{U}\bm{\beta}_{1} + \beta_{2} \bm{v}_{0}$. Then $\varphi_{0}(\bm{v}) = \varphi_{0}(\bm{w})$ leads to the relation $\bm{\beta}_{1} = \frac{1+\beta_{2}}{1+\alpha_{2}}\bm{\alpha}_{1}$. Substituting $\left\| \bm{\alpha}_{1} \right\|_{2}^{2} + \alpha_{2}^{2} = 1,\; \left\| \bm{\beta}_{1} \right\|_{2}^{2} + \beta_{2}^{2} = 1$ for the relation, $(1+\beta_{2})(\alpha_{2} - \beta_{2}) = 0$ holds. Because $\beta_{2} \neq -1$, $\alpha_{2} = \beta_{2}$ is satisfied and $\bm{v} = \bm{w}$ holds.

It is easy to observe that $\varphi_{0}$ is surjective. For every $\bm{\beta}\in\mathbb{R}^{d-k}$, let $\bm{\alpha}_{1} = \frac{2}{1+\left\| \bm{\beta} \right\|_{2}^{2}} \bm{\beta},\; \alpha_{2} = \frac{1-\left\| \bm{\beta} \right\|_{2}^{2}}{1+\left\| \bm{\beta} \right\|_{2}^{2}},\; \bm{v}=\bm{U}\bm{\alpha}_{1} + \alpha_{2} \bm{v}_{0}$. Then, $\bm{v}$ belongs to $M$, $\left\| \bm{\alpha}_{1} \right\|_{2}^{2} + \alpha_{2}^{2} = 1$ is satisfied, and $\varphi_{0}(\bm{v}) = \bm{\beta}$ holds.

We then show that $\varphi_{0}$ is continuous on the domain $M$. Let $f(\bm{w}) = \frac{1}{1+\bm{v}_{0}^{\top}\bm{w}}$ and $g(\bm{w}) = \bm{U}^{\top}\bm{w}$; then, $\varphi_{0}(\bm{w}) = f(\bm{w}) g(\bm{w})$. It is obvious that the function $g$ is continuous, because it is a finite-dimensional linear map. Thus, we provide a proof that the function $f$ is continuous on domain $M$, which is sufficient to show that $^{\forall} \bm{v}\in M,\; ^{\forall} \epsilon \in \left(0,1\right],\; ^{\exists} \delta > 0,\; ^{\forall} \bm{w}\in M,\; \left\|\bm{w} - \bm{v}\right\|_{2} < \delta \Rightarrow \left| f(\bm{w}) - f(\bm{v}) \right| < \epsilon$.
For every $\bm{v}\in M,\; \epsilon\in\left(0,1\right]$, let $\delta = \frac{\left| 1+\bm{v}_{0}^{\top}\bm{v} \right|^2}{3} \epsilon$. Then, every $\bm{w}\in M$ such that $\left\| \bm{w}-\bm{v} \right\|_{2} < \delta$ satisfies
\begin{align*}
  &\left| \left| 1+\bm{v}_{0}^{\top}\bm{w} \right| - \left| 1+\bm{v}_{0}^{\top}\bm{v} \right| \right|
    \leq \left| \bm{v}_{0}^{\top}(\bm{w} - \bm{v}) \right|
    < \frac{\left| 1+\bm{v}_{0}^{\top}\bm{v} \right|^2}{3}\epsilon
\\
  &\Rightarrow \quad
  \frac{ \left\| \bm{w} - \bm{v} \right\|_{2} }{\left| 1+\bm{v}_{0}^{\top}\bm{w} \right| \left| 1+\bm{v}_{0}^{\top}\bm{v} \right|}
  < \frac{\epsilon}{3-\left| 1+\bm{v}_{0}^{\top}\bm{v} \right|\epsilon}
.
\end{align*}
The condition $\epsilon\in\left(0,1\right]$ leads to $3-\left| 1+\bm{v}_{0}^{\top}\bm{v} \right|\epsilon \geq 1$. Therefore,
\begin{align*}
  \left| f(\bm{w}) - f(\bm{v}) \right| \leq \frac{ \left\| \bm{w} - \bm{v} \right\|_{2} }{\left| 1+\bm{v}_{0}^{\top}\bm{w} \right| \left| 1+\bm{v}_{0}^{\top}\bm{v} \right|} < \epsilon
,
\end{align*}
holds. The continuity of the function $f$ implies that the function $\varphi_{0}$ is continuous.

It is easy to observe that $\varphi_{0}^{-1}$ is the inverse function of $\varphi_{0}$ because $^{\forall} \bm{v}\in M,\; (\varphi_{0}^{-1} \circ \varphi_{0})(\bm{v}) = \bm{v}$ and $^{\forall} \bm{\beta}\in\mathbb{R}^{d-k},\; (\varphi_{0} \circ \varphi_{0}^{-1})(\bm{\beta}) = \bm{\beta}$ hold.

Finally, we provide the proof that $\varphi_{0}^{-1}$ is continuous on the domain $\mathbb{R}^{d-k}$. Let $h(\bm{\beta}) = 1+\bm{\beta}^{\top}\bm{\beta},\; q(\bm{\beta}) = 1-\bm{\beta}^{\top}\bm{\beta},\; r(\bm{\beta}) = \bm{U}\bm{\beta},\; \bm{\beta}\in\mathbb{R}^{d-k}$; then, $\varphi_{0}^{-1}(\bm{\beta}) = \frac{2}{h(\bm{\beta})} r(\bm{\beta}) + \frac{q(\bm{\beta})}{h(\bm{\beta})}v_{0}$. It is sufficient to show that $h$ and $q$ are continuous on the domain $\mathbb{R}^{d-k}$. To show the continuity of $h$, it is sufficient to show that $^{\forall} \bm{\beta}\in\mathbb{R}^{d-k},\; ^{\forall} \epsilon \in \left(0,1\right],\; ^{\exists} \delta > 0,\; ^{\forall} \bm{\gamma}\in\mathbb{R}^{d-k},\; \left\| \bm{\gamma} - \bm{\beta} \right\|_{2} < \delta \Rightarrow \left| h(\bm{\gamma}) - h(\bm{\beta}) \right| < \epsilon$. For every $\bm{\beta}\in\mathbb{R}^{d-k},\; \epsilon\in\left(0,1\right]$, let $\delta = \frac{\epsilon}{\left\| \bm{\beta} \right\|_{2} + \sqrt{1 + \left\| \bm{\beta} \right\|_{2}^{2}}}$. Then, every $\bm{\gamma}\in\mathbb{R}^{d-k}$ such that $\left\| \bm{\gamma} - \bm{\beta} \right\|_{2} < \delta$ satisfies
\begin{align*}
  &\left| h(\bm{\gamma}) - h(\bm{\beta}) \right|
  \leq \left\| \bm{\gamma} - \bm{\beta} \right\|_{2}^{2} + 2\left\| \bm{\beta} \right\|_{2} \left\| \bm{\gamma} - \bm{\beta} \right\|_{2}
\\
  &< \left( \frac{\epsilon}{\left\| \bm{\beta} \right\|_{2} + \sqrt{1 + \left\| \bm{\beta} \right\|_{2}^{2}}} \right)^{2}
      + \frac{ 2\left\| \bm{\beta} \right\|_{2} \epsilon}{\left\| \bm{\beta} \right\|_{2} + \sqrt{1 + \left\| \bm{\beta} \right\|_{2}^{2}}}
,\quad
  \text{we apply $\epsilon\in\left(0,1\right]$,}
\\
  &\leq \frac{\epsilon}{ \left( \left\| \bm{\beta} \right\|_{2} + \sqrt{1 + \left\| \bm{\beta} \right\|_{2}^{2}} \right)^2}
        + \frac{ 2\left\| \bm{\beta} \right\|_{2} \epsilon}{\left\| \bm{\beta} \right\|_{2} + \sqrt{1 + \left\| \bm{\beta} \right\|_{2}^{2}}}
  = \epsilon
.
\end{align*}
The above inequalities show that the function $h$ is continuous on the domain $\mathbb{R}^{d-k}$. The continuity of $q$ is proven in the same way. The continuity of $r$ is obvious, because it is a finite-dimensional linear mapping. Therefore, it is proven that $\varphi_{0}^{-1}$ is continuous on the domain $\mathbb{R}^{d-k}$.

These results means that $M$ and $\mathbb{R}^{d-k}$ are homeomorphic and $\varphi_{0}$ is a homeomorphism.
\end{proof}



\section{Dataset Description}
The datasets adopted for evaluation are as follows. 
\begin{enumerate}
\item WBC: Breast Cancer Wisconsin (Diagnostics) dataset is a classification dataset, which has records of measurement for breast cancer cases. There are two classes, benign and malignant. The malignant class of this dataset is considered as outliers.
\item Pendigits: This dataset contains 10 classes corresponding to the digits ranging from 0 to 9, with examples created by different handwriting. Class 4, defined here as outliers, is down sampled to 20 objects only.
\item Wine: A multiclass classification dataset with 13 attributes and 3 classes. These data are the result of a chemical analysis of wines grown in the same region in Italy but derived from three different cultivars. Classes 2 and 3 are used as inliers while class 1 is down sampled to 10 instances to be used as outliers. 
\item Vertebral: The Vertebral Column dataset is a bio-medical multiclass classification dataset with six attributes, Each patient is represented in the dataset by six bio-mechanical attributes, THe class ``Abnormal'' is the majority class with 210 instances that are used as inliers and ``Normal'' is down sampled from 100 to 30 instances as outliers. 
\item Thyroid:
The thyroid disease (ann-thyroid) dataset is a three-class dataset with 6 real and 15 categorical attributes. It has 3772 training instances, with only 6 real attributes. The ``hyperfunction'' class is treated as an outlier class and the other two classes are inliers, because hyperfunction is a clear minority class.

\end{enumerate}


\begin{thebibliography}{47}

\bibitem[\protect\astroncite{Absil et~al.}{2007}]{Absil2007}
Absil, P.~A., Mahony, R., and Sepulchre, R. (2007).
\newblock {\em Optimization Algorithms on Matrix Manifolds}.
\newblock Princeton University Press.

\bibitem[\protect\astroncite{Bickel and Fr\"{u}hwirth}{2006}]{Bickel2006}
Bickel, D.~R. and Fr\"{u}hwirth, R. (2006).
\newblock On a fast, robust estimator of the mode: Comparisons to other robust
  estimators with applications.
\newblock {\em Computational Statistics \& Data Analysis}, 50(12):3500--3530.

\bibitem[\protect\astroncite{Botev et~al.}{2010}]{Botev2010}
Botev, Z.~I., Grotowski, J.~F., and Kroese, D.~P. (2010).
\newblock Kernel density estimation via diffusion.
\newblock {\em Annals of Statistics}, 38(5):2916--2957.

\bibitem[\protect\astroncite{Brooks et~al.}{2013}]{Brooks2013}
Brooks, J.~P., Dula, J.~H., and Boone, E.~L. (2013).
\newblock A pure l1-norm principal component analysis.
\newblock {\em Computational statistics \& data analysis}, 61:83--98.

\bibitem[\protect\astroncite{Cand\`{e}s et~al.}{2011}]{Candes2011}
Cand\`{e}s, E.~J., Li, X., Ma, Y., and Wright, J. (2011).
\newblock Robust principal component analysis?
\newblock {\em Journal of the ACM}, 58(3).

\bibitem[\protect\astroncite{Chac{\'o}n}{2020}]{doi:10.1111/insr.12340}
Chac{\'o}n, J.~E. (2020).
\newblock The modal age of statistics.
\newblock {\em International Statistical Review}, 88(1):122--141.

\bibitem[\protect\astroncite{Cheng}{1995}]{400568}
Cheng, Y. (1995).
\newblock Mean shift, mode seeking, and clustering.
\newblock {\em IEEE Transactions on Pattern Analysis and Machine Intelligence},
  17(8):790--799.

\bibitem[\protect\astroncite{Critchley}{1985}]{Frank1985}
Critchley, F. (1985).
\newblock Influence in principal components analysis.
\newblock {\em Biometrika}, 72(3):627--636.

\bibitem[\protect\astroncite{Croux et~al.}{2007}]{Croux2007}
Croux, C., Filzmoser, P., and Oliveira, M.~R. (2007).
\newblock Algorithms for projection-pursuit robust principal component
  analysis.
\newblock {\em Chemometrics and Intelligent Laboratory Systems},
  87(2):218--225.

\bibitem[\protect\astroncite{Croux and Ruiz-Gazen}{2005}]{Croux2005}
Croux, C. and Ruiz-Gazen, A. (2005).
\newblock High breakdown estimators for principal components: the
  projection-pursuit approach revisited.
\newblock {\em Journal of Multivariate Analysis}, 95(1):206--226.

\bibitem[\protect\astroncite{Dalenius}{1965}]{10.2307/2343439}
Dalenius, T. (1965).
\newblock The mode--a neglected statistical parameter.
\newblock {\em Journal of the Royal Statistical Society. Series A (General)},
  128(1):110--117.

\bibitem[\protect\astroncite{Ding et~al.}{2006}]{Ding2006}
Ding, C. H.~Q., Zhou, D., He, X., and Zha, H. (2006).
\newblock R1-{PCA}: rotational invariant l1-norm principal component analysis
  for robust subspace factorization.
\newblock In {\em Proceedings of the 23rd International Conference on Machine
  Learning} (vol.~148, pp.~281--288). ACM.

\bibitem[\protect\astroncite{Fukunaga and
  Hostetler}{1975}]{DBLP:journals/tit/FukunagaH75}
Fukunaga, K. and Hostetler, L. (1975).
\newblock The estimation of the gradient of a density function, with
  applications in pattern recognition.
\newblock {\em IEEE Transactions on Information Theory}, 21(1):32--40.

\bibitem[\protect\astroncite{Hubert et~al.}{2005}]{Hubert2005}
Hubert, M., Rousseeuw, P.~J., and Branden, K.~V. (2005).
\newblock ROBPCA: A new approach to robust principal component analysis.
\newblock {\em Technometrics}, 47(1):64--79.

\bibitem[\protect\astroncite{Jolliffe}{2002}]{Jolliffe2002}
Jolliffe, I. (2002).
\newblock {\em Principal component analysis}.
\newblock Springer Verlag, New York.

\bibitem[\protect\astroncite{Kemp and Silva}{2012}]{KEMP201292}
Kemp, G.~C. and Silva, J.~S. (2012).
\newblock Regression towards the mode.
\newblock {\em Journal of Econometrics}, 170(1):92--101.

\bibitem[\protect\astroncite{Kwak}{2008}]{Kwak2008}
Kwak, N. (2008).
\newblock Principal component analysis based on l1-norm maximization.
\newblock {\em IEEE Transactions on Pattern Analysis and Machine Intelligence},
  30(9):1672--1680.

\bibitem[\protect\astroncite{Lee}{1989}]{LEE1989337}
Lee, M. (1989).
\newblock Mode regression.
\newblock {\em Journal of Econometrics}, 42(3):337--349.

\bibitem[\protect\astroncite{Lerman and
  Maunu}{2018}]{DBLP:journals/pieee/LermanM18}
Lerman, G. and Maunu, T. (2018).
\newblock An overview of robust subspace recovery.
\newblock {\em Proceedings of the IEEE}, 106(8):1380--1410.

\bibitem[\protect\astroncite{Lerman et~al.}{2015}]{Lerman2015}
Lerman, G., McCoy, M.~B., Tropp, J.~A., and Zhang, T. (2015).
\newblock Robust computation of linear models by convex relaxation.
\newblock {\em Foundations of Computational Mathematics}, 15(2):363--410.

\bibitem[\protect\astroncite{Li and Chen}{1985}]{Li1985}
Li, G. and Chen, Z. (1985).
\newblock Projection-pursuit approach to robust dispersion matrices and
  principal components: Primary theory and monte carlo.
\newblock {\em Journal of the American Statistical Association},
  80(391):759--766.

\bibitem[\protect\astroncite{Meyer}{2000}]{Meyer2000}
Meyer, C.~D. (2000).
\newblock {\em Matrix analysis and applied linear algebra}.
\newblock Siam.

\bibitem[\protect\astroncite{Miyagawa et~al.}{2018}]{8451752}
Miyagawa, S., Yoshizawa, S., and Yokota, H. (2018).
\newblock Trimmed median pca for robust plane fitting.
\newblock In {\em the 25th IEEE International Conference on Image Processing} (pp.~753--757).

\bibitem[\protect\astroncite{Nie et~al.}{2014}]{Nie2014}
Nie, F., Yuan, J., and Huang, H. (2014).
\newblock Optimal mean robust principal component analysis.
\newblock In {\em Proceedings of the 31st International Conference on Machine
  Learning} (pp.~1062--1070).

\bibitem[\protect\astroncite{Ota et~al.}{2019}]{ota2019}
Ota, H., Kato, K., and Hara, S. (2019).
\newblock Quantile regression approach to conditional mode estimation.
\newblock {\em Electronic Journal of Statistics}, 13(2):3120--3160.

\bibitem[\protect\astroncite{Parzen}{1962}]{Parzen1962}
Parzen, E. (1962).
\newblock On estimation of a probability density function and mode.
\newblock {\em Annals of Mathematical Statistics}, 33(3):1065--1076.

\bibitem[\protect\astroncite{Pimentel-Alarc{\'o}n and
  Nowak}{2017}]{Pimentel2017}
Pimentel-Alarc{\'o}n, D. and Nowak, R. (2017).
\newblock Random consensus robust PCA.
\newblock {\em Electronic Journal of Statistics}, 11(2):5232--5253.

\bibitem[\protect\astroncite{Rahmani and Atia}{2017}]{Rahmani2017}
Rahmani, M. and Atia, G. (2017).
\newblock Coherence pursuit: Fast, simple, and robust subspace recovery.
\newblock In {\em Proceedings of the 34th International Conference on Machine
  Learning} (vol.~70, pp.~2864--2873).

\bibitem[\protect\astroncite{Rao}{1983}]{Rao1983}
Rao, P., editor (1983).
\newblock {\em Nonparametric Functional Estimation}.
\newblock Probability and Mathematical Statistics: A Series of Monographs and
  Textbooks. Academic Press.

\bibitem[\protect\astroncite{Sando et~al.}{2019}]{Sando2019InformationGO}
Sando, K., Akaho, S., Murata, N., and Hino, H. (2019).
\newblock Information geometry of modal linear regression.
\newblock {\em Information Geometry}, 2(1):43--75.

\bibitem[\protect\astroncite{Schmitt and Vakili}{2016}]{Schmitt2016}
Schmitt, E. and Vakili, K. (2016).
\newblock The fastHCS algorithm for robust PCA.
\newblock {\em Statistics and Computing}, 26(6):1229--1242.

\bibitem[\protect\astroncite{Scrucca}{2011}]{DBLP:journals/csda/Scrucca11}
Scrucca, L. (2011).
\newblock Model-based SIR for dimension reduction.
\newblock {\em Computational Statistics \& Data Analysis}, 55(11):3010--3026.

\bibitem[\protect\astroncite{Sheather and Jones}{1991}]{Sheather1991}
Sheather, S.~J. and Jones, M.~C. (1991).
\newblock A reliable data-based bandwidth selection method for kernel density
  estimation.
\newblock {\em Journal of the Royal Statistical Society. Series B
  (Methodological)}, 53(3):683--690.

\bibitem[\protect\astroncite{Shi et~al.}{2009}]{Shi2009}
Shi, X., Wu, Y., and Miao, B. (2009).
\newblock A note on the convergence rate of the kernel density estimator of the
  mode.
\newblock {\em Statistics \& Probability Letters}, 79(17):1866--1871.

\bibitem[\protect\astroncite{Silverman}{1981}]{silverman_mode1981}
Silverman, B.~W. (1981).
\newblock Using kernel density estimates to investigate multimodality.
\newblock {\em Journal of the Royal Statistical Society. Series B
  (Methodological)}, 43(1):97--99.

\bibitem[\protect\astroncite{Silverman}{1986}]{Silverman86}
Silverman, B.~W. (1986).
\newblock {\em Density Estimation for Statistics and Data Analysis}.
\newblock Chapman \& Hall.

\bibitem[\protect\astroncite{Terrell}{1990}]{Terrell1990}
Terrell, G.~R. (1990).
\newblock The maximal smoothing principle in density estimation.
\newblock {\em Journal of the American Statistical Association},
  85(410):470--477.

\bibitem[\protect\astroncite{Tsakiris and Vidal}{2018}]{JMLR:v19:17-436}
Tsakiris, M.~C. and Vidal, R. (2018).
\newblock Dual principal component pursuit.
\newblock {\em Journal of Machine Learning Research}, 19(1):684--732.

\bibitem[\protect\astroncite{Vieu}{1996}]{Vieu1996}
Vieu, P. (1996).
\newblock A note on density mode estimation.
\newblock {\em Statistics \& Probability Letters}, 26(4):297 -- 307.

\bibitem[\protect\astroncite{Xu et~al.}{2010a}]{DBLP:conf/colt/XuCM10}
Xu, H., Caramanis, C., and Mannor, S. (2010a).
\newblock Principal component analysis with contaminated data: The high
  dimensional case.
\newblock In {\em the 23rd Conference on Learning Theory} (pp.~490--502).

\bibitem[\protect\astroncite{Xu et~al.}{2013}]{Xu2013}
Xu, H., Caramanis, C., and Mannor, S. (2013).
\newblock Outlier-robust PCA: the high-dimensional case.
\newblock {\em IEEE Transactions on Information Theory}, 59(1):546--572.

\bibitem[\protect\astroncite{Xu et~al.}{2010b}]{Xu2010}
Xu, H., Caramanis, C., and Sanghavi, S. (2010b).
\newblock Robust PCA via outlier pursuit.
\newblock In {\em Advances in Neural Information Processing Systems} (pp.~2496--2504).

\bibitem[\protect\astroncite{Yamasaki and
  Tanaka}{2019}]{DBLP:conf/icmla/Yamasaki19}
Yamasaki, R. and Tanaka, T. (2019).
\newblock Kernel selection for modal linear regression: Optimal kernel and irls
  algorithm.
\newblock In {\em Proceedings of the 18th International Conference on Machine
  Learning and Applications}.

\bibitem[\protect\astroncite{Yang and Xu}{2015}]{Yang2015}
Yang, W. and Xu, H. (2015).
\newblock A unified framework for outlier-robust pca-like algorithms.
\newblock In {\em Proceedings of the 32st International Conference on Machine
  Learning} (pp.~484--493).

\bibitem[\protect\astroncite{Yao and Li}{2014}]{Yao2014}
Yao, W. and Li, L. (2014).
\newblock A new regression model: Modal linear regression.
\newblock {\em Scandinavian Journal of Statistics}, 41(3):656--671.

\bibitem[\protect\astroncite{Zhang and Lerman}{2014}]{Zhang2014}
Zhang, T. and Lerman, G. (2014).
\newblock A novel m-estimator for robust PCA.
\newblock {\em Journal of Machine Learning Research}, 15(1):749--808.

\bibitem[\protect\astroncite{Zhao et~al.}{2018}]{zhao2018}
Zhao, J., Yu, G., and Liu, Y. (2018).
\newblock Assessing robustness of classification using an angular breakdown
  point.
\newblock {\em Annals of Statistics}, 46:3362--3389.

\end{thebibliography}
\end{document}